\documentclass[final]{colt2020}

\usepackage{cleveref}
\usepackage{hyperref}
\usepackage{nicefrac}
\usepackage{dsfont}
\usepackage{mathtools}
\usepackage{wrapfig}
\usepackage{enumitem}

\DeclareMathOperator*{\argmin}{arg\,min}
\DeclareMathOperator*{\argmax}{arg\,max}
\DeclareMathOperator{\dif}{\textup{d}\!}

\newcommand{\cbr}[1]{\left\{#1\right\}}
\newcommand{\del}[1]{\left(#1\right)}

\newcommand{\envert}[1]{\left|#1\right|}
\newcommand{\enVert}[1]{\left\|#1\right\|}
\newcommand{\ip}[2]{\left\langle #1, #2 \right \rangle}

\newcommand{\COMMENT}[3]{\marginpar{{\upshape\color{#1}\textbf{[#2]}}}{\upshape\color{#1}[#2: #3]}}

\def\SPAN{\textup{span}}

\def\Pip{\Pi_\perp}
\def\R{\mathbb{R}}
\def\cR{\mathcal{R}}
\def\nR{\nabla\cR}
\def\nf{\nabla f}

\def\barw{\bar{\vw}}
\def\barv{\bar{\vv}}
\def\baru{\bar{\vu}}
\def\hgamma{{\hat \gamma}}
\def\bgamma{{\bar \gamma}}
\def\hu{\hat{\vu}}

\def\lexp{\ell_{\exp}}
\def\lrec{\ell_{\textup{recip}}}

\def\ddefloop#1{\ifx\ddefloop#1\else\ddef{#1}\expandafter\ddefloop\fi}
\def\ddef#1{\expandafter\def\csname v#1\endcsname{\ensuremath{\boldsymbol{#1}}}}
\ddefloop ABCDEFGHIJKLMNOPQRSTUVWXYZabcdefghijklmnopqrstuvwxyz\ddefloop

\newenvironment{proofof}[1]{\begin{proof}\textbf{(of {#1})}}{\end{proof}}

\title[Gradient descent follows the regularization path for general losses]{Gradient descent follows the regularization path for general losses}

\usepackage{times}

\coltauthor{\Name{Ziwei Ji} \Email{ziweiji2@illinois.edu}\\
  \addr University of Illinois, Urbana-Champaign
  \AND
  \Name{Miroslav Dud\'ik} \Email{mdudik@microsoft.com}\\
  \addr Microsoft Research, New York, NY
  \AND
  \Name{Robert E. Schapire} \Email{schapire@microsoft.com}\\
  \addr Microsoft Research, New York, NY
  \AND
  \Name{Matus Telgarsky} \Email{mjt@illinois.edu}\\
  \addr University of Illinois, Urbana-Champaign
}

\makeatletter
\def\set@curr@file#1{\def\@curr@file{#1}} \makeatother

\begin{document}

\maketitle

\begin{abstract}
    Recent work across many machine learning disciplines
has
    highlighted that standard descent methods, even without explicit
    regularization, do not merely minimize the training error, but also
    exhibit an \emph{implicit bias}.
This bias is typically towards a certain regularized solution, and
    relies upon the details of the learning process,
    for instance the use of the cross-entropy loss.

    In this work, we show that for empirical risk minimization over linear
    predictors with \emph{arbitrary} convex, strictly decreasing losses, if the
    risk does not attain its infimum, then the gradient-descent path and the
    \emph{algorithm-independent} regularization path converge to the same
    direction (whenever either converges to a direction).
    Using this result, we provide a justification for the widely-used
    exponentially-tailed losses (such as the exponential loss or the logistic
    loss):
while this convergence to a direction for exponentially-tailed losses is
    necessarily to the maximum-margin direction, other losses such as
    polynomially-tailed losses may induce convergence to a direction
    with a poor margin.
\end{abstract}


\section{Introduction}\label{sec:intro}

A central problem in machine learning is \emph{overfitting}, where a predictor
performs well on training data, but poorly on testing data.
A direct way to mitigate overfitting is to add an \emph{explicit} regularizer,
such as an $\ell_1$ or $\ell_2$ penalty on the model parameters.
Another approach, achieving strong empirical results in modern models
with many parameters \citep{rethinking},
is to exploit the \emph{implicit} regularization exhibited by common descent methods,
such as coordinate descent \citep{boosting_margin}
and gradient descent \citep{soudry_linear},
simply by running them a long time with no explicit regularization.

In fact, as will be explored in this work, there is a strong relationship
between implicit and explicit regularization.
For example, coordinate-descent iterates under exponential loss minimization
(or, equivalently, \emph{AdaBoost iterates}, see \citealp{freund_schapire_adaboost}) and
$\ell_1$-regularized solutions are both biased towards $\ell_1$-maximum-margin
solutions \citep{zhang_yu_boosting,margins_shrinkage_boosting,rosset,zhao_yu}.
Similarly, gradient-descent iterates under exponential or logistic loss minimization
and the corresponding $\ell_2$-regularized solutions are both biased towards $\ell_2$-maximum-margin
solutions \citep{soudry_linear,riskparam}.

The preceding methods, which rose to prominence for their empirical performance,
all shared a curious property:
an insistence upon a loss with \emph{exponential tails},
such as the exponential loss or the logistic loss.
This is an odd coincidence, as the classical theory of classification
performance of convex losses indicates a wide variety should work well, in both
theory and practice \citep{bartlett_jordan_mcauliffe,zhang_convex_consistency}.
This leads to the central question of this work:
\begin{quote}
  \emph{For general convex decreasing losses, what is the relationship between
  gradient descent iterates and the regularized solutions?}
\end{quote}

\begin{figure}[t]
    \subfigure[Zoomed in.]{
        \label{fig:intro:1}
        \includegraphics[width = 0.47\textwidth]{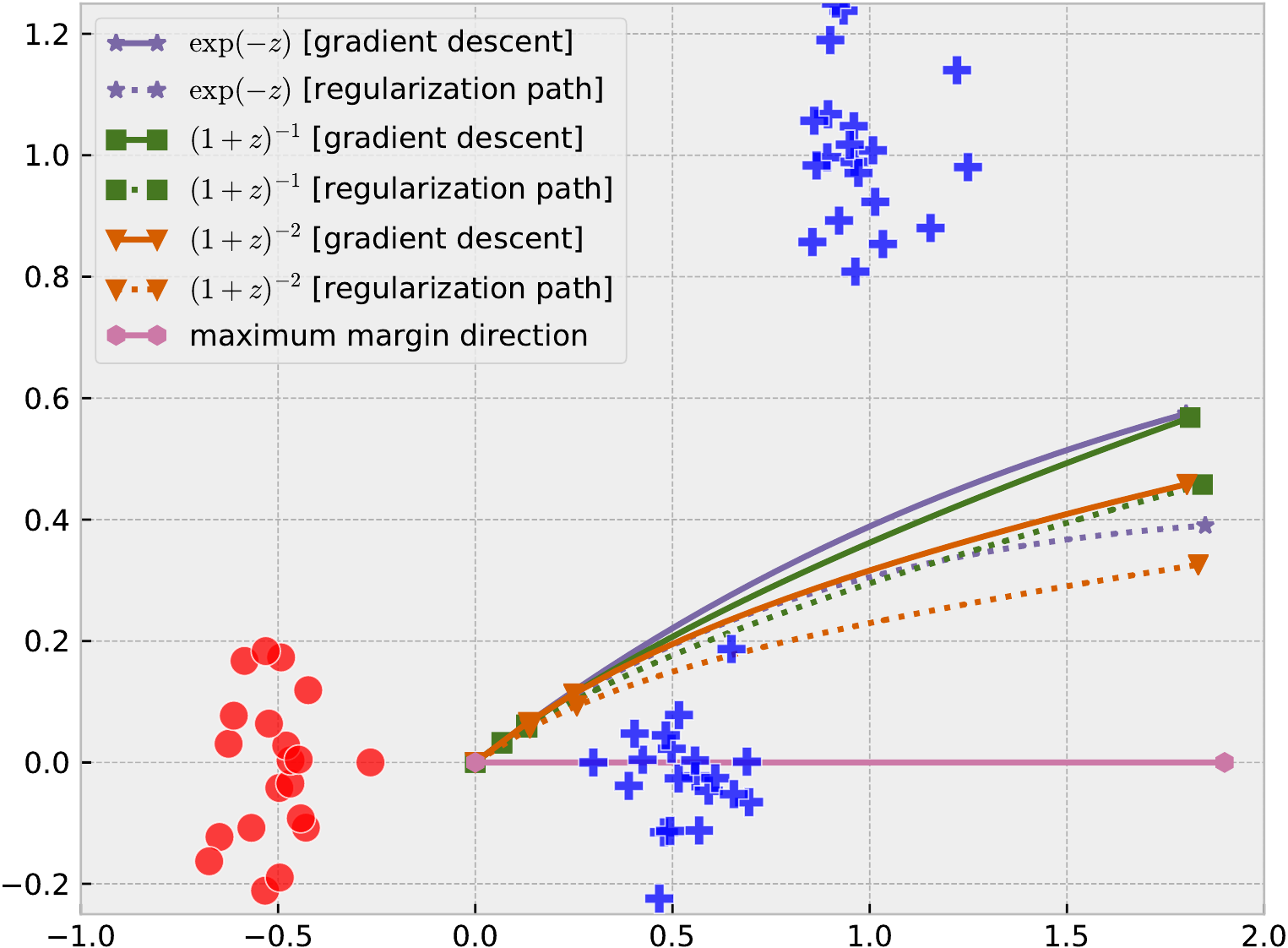}
    }
    \hfill
    \subfigure[Zoomed out.]{
        \label{fig:intro:2}
        \includegraphics[width = 0.47\textwidth]{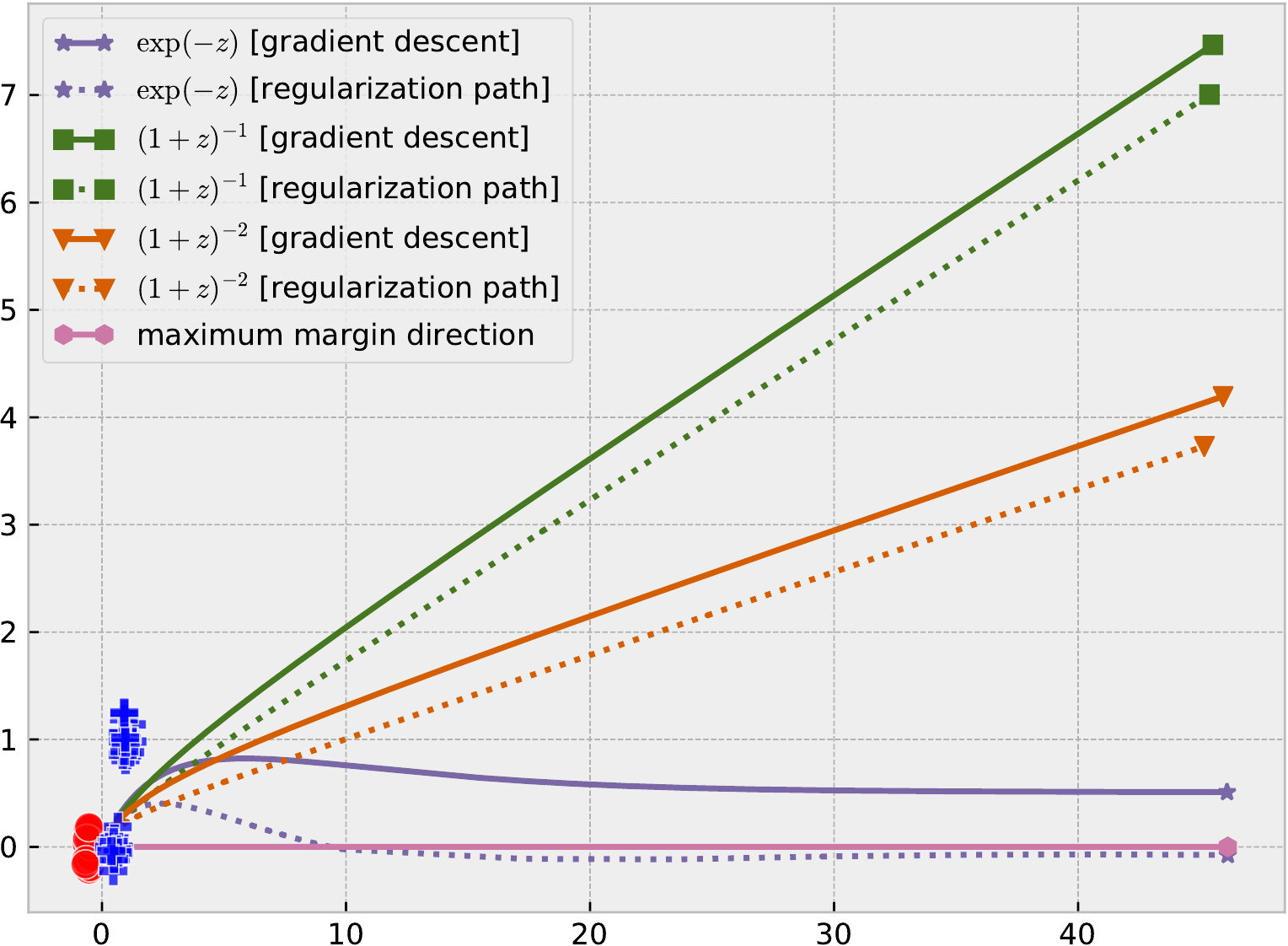}
    }
    \caption{
        Behavior of gradient descent and regularization path for three losses:
        the exponential loss $\exp(-z)$,
        and two polynomially-tailed losses $(1+z)^{-1}$ and $(1+z)^{-2}$
        (with a quadratic extension along $z<0$ for smoothness).
        The data has one negative (red) point cloud, and two positive
        (blue) point clouds; the upper positive cloud pulls the
        predictors trained with polynomially-tailed losses
        away from the maximum-margin direction,
        which points straight to the right.
}
    \label{fig:intro}
\end{figure}

This work focuses on gradient descent and $\ell_2$ regularization.
Before describing the formal results, we demonstrate on a concrete example
the trends we would like to capture. \Cref{fig:intro} shows
the path followed by gradient descent and the \emph{regularization path},
obtained by taking the regularization weight down to $0$, for three separate losses
on the same data set, consisting of the three depicted point clouds.
Zooming in on the data as in \Cref{fig:intro:1}, the behavior is unclear.
Zooming out in \Cref{fig:intro:2}, however, a trend emerges:
for each loss, its gradient-descent path and regularization path asymptotically
follow the same direction.  Moreover, the choice of loss function may lead
to a different convergent direction, and only the exponential loss
converges to the maximum-margin direction.\looseness=-1

\subsection{Contributions}

The goal of this work is to pin down the relationship between gradient-descent
paths and regularization paths for linear predictors, but only assuming the
losses are convex and strictly decreasing.

Definitions will be mostly deferred, but to summarize the main results, a bit of
notation is needed.
Throughout, $\cR$ will denote the empirical risk, and $(\vw_t)_{t\geq 0}$ will
denote gradient descent iterates given by
\begin{align}\label{eq:gd}
    \vw_{t+1} := \vw_t - \eta \nR(\vw_t),
\end{align}
where $\eta>0$ is a sufficiently small but constant step size.
Meanwhile, $\barw(B)$ will denote the regularized solution with $\ell_2$ norm
$B$; concretely,
\begin{align}\label{eq:reg}
    \barw(B):=\argmin_{\|\vw\|\le B}\cR(\vw),
\end{align}
and the \emph{regularization path} denotes the curve followed by $\barw$ as $B$
varies, meaning $(\barw(B))_{B\geq 0}$.
Choosing regularized rather than constrained solutions does not change our
results regarding the regularization path; moreover, in either case, the paths
are algorithm-independent.

As in \Cref{fig:intro}, this work is in the setting where the empirical risk $\cR$
does not attain its infimum, and consequently (as verified in \Cref{sec:if}),
both $\|\vw_t\|\to\infty$ and $\|\barw(B)\|\to\infty$.
As will be shown in \Cref{sec:sep,sec:non_sep}, with strictly decreasing losses,
$\cR$ does not attain its infimum if the training set has a nonempty
``separable'' part; it is also true in the cases of AdaBoost and deep networks,
where perfect classification is possible (cf. \Cref{sec:rw}).
Since the norms grow unboundedly,
to compare $\vw_t$ and $\barw(B)$, this work compares the directions to which they converge:
namely, $\smash[b]{\lim_{t\to\infty} \frac {\vw_t}{\|\vw_t\|}}$ and $\lim_{B\to\infty} \frac {\barw(B)}{B}$,
when the limits exist.
Since we use linear classifiers here, this normalization does not affect their
(binary) predictions.

Our core contribution can be summarized as follows.

\begin{theorem}[Coarsening of \Cref{fact:conv_dir_if,fact:conv_dir,fact:conv_dir_gen}]\label{fact:coarse}
    Suppose the loss function is convex, strictly decreasing to $0$,
    the empirical risk $\cR$ does not attain its infimum,
    and the step size $\eta>0$ is sufficiently small (as discussed in \Cref{sec:if}).
    Then $\lim_{t\to\infty}\frac{\vw_t}{\|\vw_t\|} = \lim_{B\to\infty}\frac{\bar \vw(B)}{B}$
    whenever either limit exists.
\end{theorem}

In words, \Cref{fact:coarse} states that if either the gradient-descent path or
the algorithm-independent regularization path converge to a direction,
then both of them converge to the same direction.
In more detail, our full contributions and the paper organization are as follows.

\Cref{sec:if} shows that if the gradient-descent path converges to a direction,
then the regularization path converges to the same direction.
Interestingly, this proof holds for general convex functions not attaining their
infimum, and does not require any properties of the risk.

\Cref{sec:sep} focuses on the case of \emph{linearly separable data}.
The primary effort is in showing the converse to \Cref{sec:if} in this setting,
namely that if the regularization path converges to a direction, then the
gradient-descent path converges to the same direction.
This section also establishes that exponentially-tailed losses (cf.~eq.~\ref{eq:exp_tail})
all converge to the same maximum-margin direction,
that polynomially-tailed losses (cf.~eq.~\ref{eq:poly_tail}) converge to a
direction but may only achieve a poor margin, and lastly that for general losses
the iterates may fail to converge to a direction.

\Cref{sec:non_sep} completes the picture in the case of general data which is
potentially not linearly separable: that is, if the empirical risk does not
attain its infimum, and if the regularization path converges to a direction, then
the gradient-descent path converges to the same direction.
This setting introduces significant technicalities, but also comes with
interesting refinements: while gradient descent and the regularization path do
not converge to a point (only to a direction, as in \Cref{fig:intro}) in this
nonseparable setting, it is possible to show convergence to a point over a certain subspace.

We provide concluding remarks and open problems in \Cref{sec:open}.

\subsection{Related work}\label{sec:rw}

Arguably, the earliest relevant literature is the introduction of
the support vector machine (SVM),
which utilizes explicit regularization to select maximum margin classifiers
\citep{vapnik}---the property that was eventually tied to generalization performance
\citep{STBWA1998,bartlett_margin}.
This use of explicit regularization is significantly different from the setup here:
there, the loss is hinge loss (which attains 0) and the regularization level is constant,
whereas here,
the loss necessarily asymptotes to $0$, and the regularization level is also taken to $0$.
In a concrete sense, exponential losses with this decaying regularization behave asymptotically
like the SVM, and this analogy was used explicitly in the aforementioned gradient descent
proof of \citet{soudry_linear}.
Turning back to descent methods, the original use of margins was in the analysis of
perceptron \citep{novikoff}, however there is no implicit bias: the method terminates with
0 classification error, but no reasonable lower bound can be placed on the achieved margin.

The first concrete studies showing an implicit bias of descent methods were for the
$\ell_1$-regularized case.
Coordinate descent, when paired with the exponential loss, is implicitly biased
towards $\ell_1$-regularized solutions.
This observation is the result of separate lines of work on descent
methods and on regularization methods.
On one hand, AdaBoost was shown to exhibit \emph{positive margins}, meaning
its predictions are not only correct, but in a certain sense robust
\citep{boosting_margin};
indeed, with some further care on the descent step sizes, AdaBoost finds
maximum-margin solutions \citep{zhang_yu_boosting,margins_shrinkage_boosting}.
On the other hand, the $\ell_1$-regularized solutions also converge to
maximum-margin solutions as regularization strength is taken to $0$
\citep{rosset,zhao_yu}.\looseness=-1

Another line of research has shown that gradient descent, when paired with the exponential or
logistic loss, converges to $\ell_2$-regularized solutions.
This was first established for linear methods when the data is linearly
separable \citep{soudry_linear}, meaning there exists a linear predictor which
perfectly labels all data, but has since been extended to linear predictors on
nonseparable data \citep{riskparam}.
\citet{soudry_linear} and \citet{riskparam} only handled exponentially-tailed
losses, while in this paper we prove results for general losses and do not
require separability.\looseness=-1

The implicit bias of gradient descent has also been studied for linear
convolutional networks \citep{soudry_convolutional}, deep linear networks \citep{deeplinear_alignment}, and homogeneous networks \citep{kaifeng_margin},
where empirical results seem to suggest such a bias exists \citep{behnam__deeplearning_bias,spec}.
Similarly to the situation with AdaBoost, there is a variety of results
focusing purely on explicitly-regularized methods
\citet{wei_regularization_matters}.

As a final brief remark,
implicit bias and margins have been extended beyond standard classification
settings, for instance to adversarial training
\citep{adversarial_margin,li_adv}.

\section{Convergence of gradient descent implies convergence of regularization path}
\label{sec:if}

In this section we show one direction of the equivalence, which holds in a more
general setting.

Given a differentiable convex function $f:\R^d\to\R$ (not necessarily the empirical risk)
and an $\ell_2$-norm bound $B$, the regularized solution is defined as
\begin{align}\label{eq:reg_f}
    \barw(B):=\argmin_{\|\vw\|\le B}f(\vw).
\end{align}
Note that $\barw(B)$ is not unique in general, but we still have
$\lim_{B\to\infty}f\del{\barw(B)}=\inf_{\vw\in\R^d}f(\vw)$, as is often the case
when working with unregularized losses.
In this paper we are particularly interested in the case where the infimum of
$f$ is not attained.
In that case $\barw(B)$ is uniquely defined, because the set of minimizers is convex
and contained in the surface of the $\ell_2$ ball, and thus consists of exactly
one point due to the curvature of $\ell_2$ balls.
An example of a function $f$ that does not attain the infimum is $e^{-z}$:
its infimum is $0$, which is not attained by any
$z\in\R$.
A more interesting example is an empirical risk with a nonempty separable part,
which will be introduced in \Cref{sec:sep,sec:non_sep}.\looseness=-1

We minimize $f$ using gradient descent, meaning
\begin{align}\label{eq:gd_f}
    \vw_{t+1}:=\vw_t-\eta\nf(\vw_t).
\end{align}
Its basic properties are summarized in \Cref{fact:gd}.
If there exists a small step size which ensures decreasing function values, then
gradient descent on $f$ can minimize the function value to its infimum;
moreover, if the infimum of $f$ is not attained, then gradient descent iterates
go to infinity.
\begin{lemma}\label[lemma]{fact:gd}
    Given a convex differentiable function $f:\R^d\to\R$, suppose the step size
    $\eta$ satisfies
    \begin{align}\label{eq:gd_func}
        f(\vw_{t+1})-f(\vw_t)\le-\frac{\eta}{2}\enVert{\nf(\vw_t)}^2
    \end{align}
    for all $t\ge0$.
    Then for any $\vw\in\R^d$,
    \begin{align}\label{eq:gd_sq_dist}
        \|\vw_{t+1}-\vw\|^2\le\|\vw_t-\vw\|^2+2\eta\del{f(\vw)-f(\vw_{t+1})},
    \end{align}
    and thus $\|\vw_{t+1}-\vw\|\le\|\vw_t-\vw\|$ as long as
    $f(\vw)\le f(\vw_{t+1})$.
    Consequently,
    \begin{align*}
        \lim_{t\to\infty}f(\vw_t)=\inf_{\vw\in\R^d}f(\vw),
    \end{align*}
    which implies $\lim_{t\to\infty}\|\vw_t\|=\infty$ if the infimum of $f$ is not
    attained.
\end{lemma}

\begin{remark}
  The step size condition in \cref{eq:gd_func} holds if $f$ is (globally)
  $\beta$-smooth and $\eta\le1/\beta$.
  There are also standard situations where $f$ merely obeys local smoothness
  over its sublevel sets;
  see for example \cref{eq:exp_risk}, which considers empirical risk minimization
  with the exponential loss.
\end{remark}

Below is our main result of this section.
\begin{theorem}\label[theorem]{fact:conv_dir_if}
    Consider the gradient descent iterates $(\vw_t)_{t\ge0}$ given by
    \cref{eq:gd_f}, and the regularized solutions $(\barw(B))_{B\ge0}$ given by
    \cref{eq:reg_f}.
    Suppose $f$ is convex, differentiable, bounded below by $0$, and has an
    unattained infimum, and the step size $\eta$ satisfies \cref{eq:gd_func} and
    $\eta\le1/\del{2f(\vw_0)}$.
    If $\lim_{t\to\infty}\vw_t/\|\vw_t\|=\baru$ for some unit vector $\baru$, then
    also $\lim_{B\to\infty}\barw(B)/B=\baru$.
\end{theorem}

The full proof of \Cref{fact:conv_dir_if} is given in \Cref{app_sec:if}.
Here we sketch the main arguments.
The key property used in the proof is \cref{eq:gd_sq_dist}.
Note that given any $B>0$, by the definition of $\barw(B)$, as long as
$\|\vw_t\|,\|\vw_{t+1}\|\le B$, it holds that
$\enVert{\vw_{t+1}-\barw(B)}\le\enVert{\vw_t-\barw(B)}$.
In other words, the distance from the gradient-descent path to $\barw(B)$ is
nonincreasing within the ball $\cbr{\vw:\|\vw\|\le B}$.

Suppose for some $\epsilon>0$, there exists arbitrarily large $B$ with
$\enVert{\frac{\barw(B)}{B}-\baru}>\epsilon$.
By Euclidean geometry, we can show that
\begin{align*}
    \enVert{B\baru-\barw(B)}-\enVert{\ip{\barw(B)}{\baru}\baru-\barw(B)}>\frac{B\epsilon^3}{8}.
\end{align*}
By the assumption, if $\|\vw_t\|$ is large enough, then $\vw_t/\|\vw_t\|$ and $\baru$
can be arbitrarily close.
The idea is then to find two gradient descent iterates $\vw_{t_1}$ and $\vw_{t_2}$,
where $t_1<t_2$, and $\vw_{t_1}$ is close to $\ip{\barw(B)}{\baru}\baru$,
and $\vw_{t_2}$ is close to $B\baru$.
It then follows that $\enVert{\vw_{t_2}-\barw(B)}>\enVert{\vw_{t_1}-\barw(B)}$,
which violates \cref{eq:gd_sq_dist}.

\section{Convergence to a direction for the linearly separable case}
\label{sec:sep}

In the remainder of the paper, we consider binary classification
with a training set $\{(\vx_i,y_i)\}_{i=1}^n$, where $\vx_i\in\R^d$ and
$y_i\in\{-1,+1\}$, and we assume $\|\vx_i\|\le1$ without loss of generality.
We use a linear classifier $\vw\in\R^d$, which is learned by minimizing the
empirical risk
\begin{align*}
    \cR(\vw):=\frac{1}{n}\sum_{i=1}^{n}\ell\del{y_i \langle \vw,\vx_i\rangle},
\end{align*}
where the loss function $\ell$ is assumed to be convex, differentiable, and
strictly decreasing to $0$, such as the logistic loss $\ln(1+e^{-z})$.

In this section, we assume that the training data is linearly separable: there
exists a unit vector $\vu$ and some $\gamma>0$ such that
$y_i \langle \vu,\vx_i\rangle\ge\gamma$ for all $1\le i\le n$.
Results in this section can be extended to the general case with no assumption
on the training data, as we will do in \Cref{sec:non_sep}.

Linear separability and a strictly decreasing loss imply that the infimum of
$\cR$ is not attained, and thus \Cref{fact:conv_dir_if} can be applied.
However, we can show a stronger result: the gradient-descent path
converges to a direction if and only if the regularization path converges
to (the same) direction.
\begin{theorem}\label{fact:conv_dir}
    Consider the gradient descent iterates $(\vw_t)_{t\ge0}$ given by
    \cref{eq:gd}, and the regularized solutions $(\barw(B))_{B\ge0}$ given by
    \cref{eq:reg}.
    Suppose the data is linearly separable, and the step size satisfies
    $\eta\le1/\del{2\cR(\vw_0)}$ and
    \begin{align}\label{eq:gd_risk}
        \cR(\vw_{t+1})-\cR(\vw_t)\le-\frac{\eta}{2}\enVert{\cR(\vw_t)}^2
    \end{align}
    for all $t\ge t_0$.
    Then $\lim_{t\to\infty}\vw_t/\|\vw_t\|$ exists if and only if
    $\lim_{B\to\infty}\barw(B)/B$ exists, and when they exist they are the same.
\end{theorem}

\begin{remark}
    It can be verified that if the loss function $\ell$ is $\beta$-smooth, then
    so is the empirical risk function $\cR$, and \cref{eq:gd_risk} holds if
    $\eta\le1/\beta$.
    However, it may still hold for a loss function which is not globally smooth.
    For example, for the exponential loss $e^{-z}$, Lemma 3.4 of \citet{riskparam}
    ensures that
    \begin{align}\label{eq:exp_risk}
        \cR(\vw_{t+1})-\cR(\vw_t)\le-\eta\del{1-\frac{\eta\cR(\vw_t)}{2}}\enVert{\nR(\vw_t)}^2
    \end{align}
    as long as $\eta\cR(\vw_t)\le1$.
    Therefore, \cref{eq:gd_risk} holds as long as $\eta\le1/\cR(\vw_0)$.
\end{remark}

The ``if'' part of \Cref{fact:conv_dir} follows directly from
\Cref{fact:conv_dir_if}.
Next we give a proof sketch of the ``only if'' part of \Cref{fact:conv_dir}; the
full proof is given in \Cref{app_sec:sep}.

In the remainder of this section, we assume that
$\lim_{B\to\infty}\barw(B)/B=\baru$ for some unit vector~$\baru$, and define
its margin as
\begin{align*}
    \bgamma:=\min_{1\le i\le n}y_i \langle\baru,\vx_i\rangle.
\end{align*}
Moreover, the maximum margin $\hgamma$ and the maximum-margin solution $\hu$ are
defined as
\begin{align*}
    \hgamma:=\max_{\|\vu\|=1}\min_{1\le i\le n}y_i \langle \vu,\vx_i\rangle,\quad\textrm{and}\quad\hu:=\argmax_{\|\vu\|=1}\min_{1\le i\le n}y_i \langle \vu,\vx_i\rangle.
\end{align*}
We first show that $\bgamma$ is always positive.
\begin{lemma}\label[lemma]{fact:baru_margin}
    It holds that $\bgamma\ge\hgamma^2/(2n)>0$, where $\hgamma$ is the maximum
    margin.
\end{lemma}

\begin{remark}
    \Cref{fact:baru_margin} gives a worst-case lower bound on margin, which holds
    for an arbitrary decreasing loss.
    The proof technique can be adapted to a specific loss function.
    For example, if the loss function has a polynomial tail $az^{-b}$, then
    $\lim_{B\to\infty}\barw(B)/B$ exists (cf. \Cref{fact:poly_baru}), and we can
    prove an $\Omega(n^{-1/(b+1)})$ lower bound on margin.
    Moreover, there exists a dataset on which this lower bound is tight (cf.
    \Cref{fact:baru_margin_ub}).
\end{remark}

Here is a proof sketch of \Cref{fact:baru_margin}.
The starting point is the property that $\barw(B)$ and $\nR\del{\barw(B)}$ are
collinear, meaning
\begin{align}\label{eq:barw_col}
    -\ip{\frac{\barw(B)}{B}}{\nR\del{\barw(B)}}=\enVert{\nR\del{\barw(B)}},
\end{align}
which is a consequence of the first-order optimality conditions.
Next, by the chain rule, the left hand side of \cref{eq:barw_col} is naturally
related to the margin of $\barw(B)/B$:
\begin{align}\label{eq:barw_col_lhs}
    -\ip{\frac{\barw(B)}{B}}{\nR\del{\barw(B)}}=\frac{1}{n}\sum_{i=1}^{n}-\ell'\del{\ip{\barw(B)}{y_i\vx_i}}\ip{\frac{\barw(B)}{B}}{y_i\vx_i},
\end{align}
while the right hand side of \cref{eq:barw_col} can be bounded using the
Cauchy-Schwarz inequality and the maximum-margin solution $\hu$:
\begin{align}\label{eq:barw_col_rhs}
    \enVert{\nR\del{\barw(B)}}\ge\ip{-\nR\del{\barw(B)}}{\hu}
    &\ge \frac{1}{n}\sum_{i=1}^{n}-\ell'\del{\ip{\barw(B)}{y_i\vx_i}}\hgamma.
\end{align}
If $\bgamma<\hgamma^2/(2n)$ , then since the regularization path converges to
$\baru$, the margin of $\barw(B)/B$ is no larger than $\hgamma^2/(2n)$ for all
large $B$.
To ensure that \cref{eq:barw_col_lhs} is upper bounded by \cref{eq:barw_col_rhs}
would require that
\begin{align*}
    -\ell'(B\hgamma)\ge-\ell'\del{\frac{B\hgamma^2}{2n}}\frac{\hgamma}{2n}.
\end{align*}
This would in turn imply $\int_0^\infty-\ell'(z)\dif z=\infty$, a contradiction.

Next we can show that to minimize the risk, it is almost optimal to move along
the direction of~$\baru$, thanks to its positive margin.
\begin{lemma}\label[lemma]{fact:baru_dec}
  Given any $\alpha>0$, there exists $\rho(\alpha)>0$, such that for any $\vw$
  with $\|\vw\|>\rho(\alpha)$,
  it holds that
  \begin{align*}
    \cR\big((1+\alpha)\enVert{\vw}\baru\big)\le\cR(\vw).
  \end{align*}
\end{lemma}

To prove \Cref{fact:baru_dec}, first note that by definition
$\cR\big(\barw(\|\vw\|)\big)\le\cR(\vw)$, and thus it is enough to show that
$\cR\big((1+\alpha)\enVert{\vw}\baru\big)\le\cR\big(\barw(\|\vw\|)\big)$.
This is true if for all $1\le i\le n$,
\begin{align*}
    y_i\ip{(1+\alpha)\enVert{\vw}\baru}{\vx_i}\ge y_i\ip{\barw(\|\vw\|)}{\vx_i},\quad\textrm{i.e.,}\quad(1+\alpha)y_i \langle\baru, \vx_i\rangle\ge y_i\ip{\frac{\barw\del{\|\vw\|}}{\|\vw\|}}{\vx_i}.
\end{align*}
Since $y_i\langle\alpha\baru,\vx_i\rangle\ge\alpha\bgamma$ and $\|\vx_i\|\le1$, we
only need to choose $\|\vw\|$ large enough such that
\begin{align*}
    \enVert{\baru-\frac{\barw(\|\vw\|)}{\|\vw\|}}\le\alpha\bgamma.
\end{align*}

Now we are ready to prove the ``only if'' part of \Cref{fact:conv_dir}.
The full proof appears in \Cref{app_sec:sep}, but is a bit cumbersome in our discrete-time
setting; here we will illustrate the idea with the \emph{gradient flow},
meaning $\eta \to 0$ and $\dot \vw_t := \dif \vw_t / \dif t = -\nR(\vw_t)$.
For any $\alpha>0$, due to $\|\vw_t\|\to\infty$ and
\Cref{fact:baru_dec}, we can choose $t_0$ large enough so that
$\cR\del{(1+\alpha)\|\vw_t\| \baru} \leq \cR(\vw_t)$ for all $t\geq t_0$.
By convexity,
\[
  0 \geq \cR\big((1+\alpha)\|\vw_t\| \baru\big) - \cR(\vw_t)
  \geq
\ip{\dot \vw_t}{\vw_t - (1+\alpha)\|\vw_t\|\baru},
\]
which rearranges to
\[
  \ip{\dot \vw_t}{\baru}
  \geq
  \del{\frac 1{1+\alpha}}\ip{\dot \vw_t}{\frac {\vw_t}{\|\vw_t\|}}
  =
  \del{\frac 1{1+\alpha}} \frac {\dif}{\dif t}{\|\vw_t\|}.
\]
For any $t_1 \geq t_0$, integrating both sides along $[t_0,t_1]$ gives
\[
  \ip{\vw_{t_1} - \vw_{t_0}}{\baru}
  = \ip{\int_{t_0}^{t_1} \dot\vw_t\dif t}{\baru}
  \geq
  \del{\frac 1 {1+\alpha}} \int_{t_0}^{t_1} \frac \dif {\dif t}\|\vw_t\|\dif t
  =
  \frac {\|\vw_{t_1}\| - \|\vw_{t_0}\|} {1+\alpha}.
\]
Dividing both sides by $\|\vw_{t_1}\|$ and applying $\liminf_{t_1\to\infty}$,
since $\liminf_{t_1\to\infty} \vw_{t_0}/\|\vw_{t_1}\| = 0$,
\[
  \liminf_{t_1\to\infty} \ip{\frac {\vw_{t_1}}{\|\vw_{t_1}\|}}{\baru}
  =
  \liminf_{t_1\to\infty} \ip{\frac {\vw_{t_1}-\vw_{t_0}}{\|\vw_{t_1}\|}}{\baru}
  \geq
  \liminf_{t_1\to\infty}
  \frac{\|\vw_{t_1}\| - \|\vw_{t_0}\|}{(1+\alpha)\|\vw_{t_1}\|}
  = \frac 1 {1+\alpha}.
\]
Since $\alpha > 0$ was arbitrary, the ``only if'' part of \Cref{fact:conv_dir} is complete.

\subsection{What does the regularization path converge to?}\label{sec:sep_dir}

\Cref{fact:conv_dir} says that the gradient-descent path and regularization path
converge to the same direction if either of them converges to a direction.
Moreover, the regularization path is independent of the optimization algorithm,
and thus easier to study.
Here are some examples where $\barw(B)/B$ converges.\looseness=-1

A classical example is that if the loss has an exponential tail, then the
regularization path converges to the maximum-margin direction (see
\citealp{rosset}, for the case of $\ell_1$ regularization).
\begin{proposition}\label[proposition]{fact:exp_baru}
    If for some $a,b>0$,
    \begin{align}\label{eq:exp_tail}
        \lim_{z\to\infty}\frac{\ell(z)}{a\exp(-bz)}=1,
    \end{align}
    then $\lim_{B\to\infty}\barw(B)/B=\hu$, where $\hu$ is the unique maximum
    margin solution.
\end{proposition}

We also prove that if the loss has a polynomial tail, then the regularization
path converges to a direction.
\begin{proposition}\label[proposition]{fact:poly_baru}
    If for some $a,b>0$,
    \begin{align}\label{eq:poly_tail}
        \lim_{z\to\infty}\frac{-\ell'(z)}{az^{-b}}=1,
    \end{align}
    then $\lim_{B\to\infty}\barw(B)/B$ exists.
\end{proposition}

However, while an exponentially-tailed loss (cf.~eq.~\ref{eq:exp_tail}) always
induces the maximum-margin direction, a polynomially-tailed loss (cf.~eq.~\ref{eq:poly_tail})
may induce a different direction:
\begin{proposition}\label[proposition]{fact:baru_margin_ub}
    For any $b>0$, consider a loss function $\ell$ which equals $z^{-b}$ for
    $z\ge1$.
    There exists a dataset on which the maximum margin is a universal constant,
    while the regularization path with $\ell$ converges to a direction which has
    margin $\Theta(n^{-1/(b+1)})$.
\end{proposition}

Lastly, note that directional convergence should not be taken for granted:
we can construct a loss function which satisfies all the
conditions in \Cref{fact:conv_dir} (i.e., convexity, monotonicity and
eq.~\ref{eq:gd_risk}) for which $\barw(B)/B$ does not converge.
The constructed loss switches between $\exp(-z)$ and
$1/z$ countably infinitely often, with the switching locations
chosen carefully so that $\barw(B)/B$ continually oscillates.
\begin{proposition}\label[proposition]{fact:no_baru}
    There exists a loss function $\ell$ which is convex, strictly decreasing to
    $0$ and $2$-smooth for which $\barw(B)/B$ does not converge.
\end{proposition}

The proofs of all results in this subsection are given in
\Cref{app_sec:sep_dir}.

\section{Convergence to a direction for the general case} \label{sec:non_sep}

In this section, we extend the preceding results to the general case of an
arbitrary training set, that might or might not be linearly separable.
The main idea is to first partition the dataset into a separable part and a
nonseparable part using the decomposition studied by \citet{riskparam} (cf.
\Cref{fact:struct} below).
Then we prove (subject to the conditions below) that the gradient-descent path
and regularization path are strongly coupled in a highly-refined sense: (1) On
the space spanned by the nonseparable part of the dataset, convergence of both
gradient descent and the regularization path is to the same unique finite point.
(2) On the space perpendicular to the nonseparable part, as in the fully
separable case, the gradient-descent path and regularization path converge to the same
direction (if either converges to a direction).

Here we define the decomposition formally.
Given a dataset $D=\{(\vx_i,y_i)\}_{i=1}^n$, we decompose it into $D_s\cup D_c$
in the following way.
For each data example $(\vx_i,y_i)$, if there exists a unit vector $\vu$ such that
$y_i \langle \vu,\vx_i\rangle>0$ and $y_j \langle \vu,\vx_j\rangle\ge0$ for all
$1\le j\le n$, then we include $(\vx_i,y_i)$ into $D_c$, otherwise we include it
into $D_s$. (The mnemonic is \emph{``s''} for strongly-convex (as justified below) and
\emph{``c''} for its complement.)
Define
\begin{align*}
    \cR_s(\vw):=\frac{1}{n}\sum_{(\vx_i,y_i)\in D_s}^{}\ell\del{y_i \langle \vw,\vx_i\rangle},\quad\textrm{and}\quad \cR_c(\vw):=\frac{1}{n}\sum_{(\vx_i,y_i)\in D_c}^{}\ell\del{y_i \langle \vw,\vx_i\rangle},
\end{align*}
and note that $\cR=\cR_s+\cR_c$.
Further define $S:=\SPAN\big(\{\vx_i:(\vx_i,y_i)\in D_s\}\big)$, and let $\Pi_S$
denote the projection onto $S$, and $\Pip$ denote the projection onto $S^\perp$.
Given $\vw\in\R^d$, let $\vw_S:=\Pi_S\vw$ and $\vw_\perp:=\Pip \vw$.
\begin{lemma}\citep[Theorem 2.1]{riskparam}\label[lemma]{fact:struct}
    The above decomposition satisfies the following properties.
    \begin{enumerate}[label=(\arabic*)]
        \item If $\ell$ is twice continuously differentiable with $\ell''>0$,
        then $\cR_s$ has compact sublevel sets over $S$, is strongly convex over
        compact subsets of $S$, and therefore has a unique minimizer $\barv$
        over~$S$.\looseness=-1

        \item $D_c$ can be linearly separated in $S^\perp$, meaning that there
        exists a unit vector $\vu\in S^\perp$ and some $\gamma>0$, such that
        $y_i \langle \vu,\vx_i\rangle\ge\gamma$ for all $(\vx_i,y_i)\in D_c$.
    \end{enumerate}
\end{lemma}

Note that for any $\vv\in S$, and any $\vu\in S^\perp$ which can separate $D_c$, it
holds that $\cR_s(\vv)=\lim_{r\to\infty}\cR(\vv+r\vu)$, and thus
$\inf_{\vw\in\R^d}\cR(\vw)=\inf_{\vv\in S}\cR_s(\vv)=\cR_s(\barv)$.
Moreover, if $D_c\ne\emptyset$, then the infimum of $\cR$ is not attained.

With the decomposition and \Cref{fact:struct}, we can state our equivalence
result for general dataset.
\begin{theorem}\label{fact:conv_dir_gen}
    Consider the gradient descent iterates $(\vw_t)_{t\ge0}$ given by
    \cref{eq:gd}, and the regularized solutions $(\barw(B))_{B\ge0}$ given by
    \cref{eq:reg}.
    Suppose $\ell$ is twice continuously differentiable with $\ell''>0$, and the
    step size $\eta\le1/\del{2\cR(\vw_0)}$ satisfies \cref{eq:gd_risk}.
    \begin{enumerate}[label=(\arabic*)]
        \item On $S$ it holds that $\lim_{t\to\infty}\Pi_S\vw_t=\barv$
        and $\lim_{B\to\infty}\Pi_S\barw(B)=\barv$.

        \item If $D_c\ne\emptyset$, then
        $\lim_{t\to\infty}\|\vw_t\|=\lim_{B\to\infty}\|\barw(B)\|=\infty$, and
        $\lim_{t\to\infty}\vw_t/\|\vw_t\|$ exists if and only if
        $\lim_{B\to\infty}\barw(B)/B$ exists, and when they exist they are the
        same and lie in $S^\perp$.
    \end{enumerate}
\end{theorem}

The convergence result on $S$ is straightforward: it follows from \Cref{fact:gd}
that $\lim_{t\to\infty}\cR(\vw_t)=\inf_{\vw\in\R^d}\cR(\vw)=\cR_s(\barv)$.
Since $\cR_s(\vw_t)\le\cR(\vw_t)$, we also have $\cR_s(\vw_t)\to\cR_s(\barv)$.
\Cref{fact:gd} also ensures that $\cR_s(\vw_t)\le\cR(\vw_t)\le\cR(\vw_0)$, and since
$\cR_s$ is strongly convex over sublevel sets, we have
$\lim_{t\to\infty}\Pi_S\vw_t=\barv$.
The proof for regularized solutions is similar.

The ``if'' part of \Cref{fact:conv_dir_gen}(2) also follows directly from
\Cref{fact:conv_dir_if}.
The limiting direction must lie in $S^\perp$ since $\Pi_S\vw_t$ is bounded due to
\Cref{fact:conv_dir_gen}(1).
Below we give a proof sketch of the ``only if'' part of
\Cref{fact:conv_dir_gen}(2), and the complete proof is given in
\Cref{app_sec:non_sep}.
The proof is similar to the purely separable case discussed in \Cref{sec:sep},
but we must also deal with the interaction between $D_s$ and $D_c$.

Assume $D_c\ne\emptyset$, and $\lim_{B\to\infty}\barw(B)/B=\baru\in S^\perp$.
Define
\begin{align*}
    \bgamma:=\min_{(\vx_i,y_i)\in D_c}y_i \langle\baru,\vx_i\rangle.
\end{align*}
Similar to the separable case, it holds that $\bgamma>0$.
\begin{lemma}\label[lemma]{fact:baru_margin_gen}
    Under the conditions of \Cref{fact:conv_dir_gen}, it holds that
    $\bgamma\ge\hgamma^2/\del{8|D_c|}>0$.
\end{lemma}

The proof of \Cref{fact:baru_margin_gen} is similar to the proof of
\Cref{fact:baru_margin}, but uses the fact that
$\Pip\barw(B)$ is collinear with $\Pip\nR\del{\barw(B)}$.

The following result extends \Cref{fact:baru_dec} to the general setting.
\begin{lemma}\label[lemma]{fact:baru_dec_gen}
    Under the conditions of \Cref{fact:conv_dir_gen}, given any $\alpha\in(0,1)$,
    there exists $\xi(\alpha)>0$, such that for any $\vw$ with
    $\cR(\vw)-\inf_{\vw\in\R^d}\cR(\vw)\le\xi(\alpha)$, it holds that
    \begin{align*}
      \cR\del{\vw_S+(1+\alpha)\enVert{\vw_\perp}\baru}\le\cR(\vw).
    \end{align*}
\end{lemma}

The proof of the ``only if'' part of \Cref{fact:conv_dir_gen}(2) is similarly
based on \Cref{fact:baru_dec_gen} and a perceptron-style analysis.
Unlike the purely separable case, the tricky part here is that $D_c$ may have a
nonzero projection onto $S$, and thus we need to deal with $\vw_{t,S}$ carefully.
Note that convexity and \Cref{fact:baru_dec_gen} ensure that for large enough
$t$,
\begin{align*}
    \ip{\nR(\vw_t)}{\vw_{t,\perp}-(1+\alpha)\enVert{\vw_{t,\perp}}\baru} & =\ip{\nR(\vw_t)}{\vw_{t,S}+\vw_{t,\perp}-\vw_{t,S}-(1+\alpha)\enVert{\vw_{t,\perp}}\baru} \\
     & =\ip{\nR(\vw_t)}{\vw_t-\vw_{t,S}-(1+\alpha)\enVert{\vw_{t,\perp}}\baru} \\
     & \ge\cR(\vw_t)-\cR\del{\vw_{t,S}+(1+\alpha)\enVert{\vw_{t,\perp}}\baru}\ge0,
\end{align*}
which implies
\begin{align*}
    \ip{-\eta\nR(\vw_t)}{\baru}\ge \frac{1}{1+\alpha}\ip{-\eta\nR(\vw_t)}{\frac{\vw_{t,\perp}}{\|\vw_{t,\perp}\|}}.
\end{align*}
The remainder of the proof is similar to the proof of \Cref{fact:conv_dir}.

\section{Concluding remarks and open problems}\label{sec:open}

We have
established that for a wide variety of losses, gradient descent and the regularization
path converge to the same direction if either of them converges to a direction, and while many losses
guarantee such convergence, the limit direction may differ across losses.

One avenue for refinement is to go back to the general studies of classification
losses \citep[e.g.,][]{bartlett_jordan_mcauliffe,zhang_convex_consistency}.
We have pointed out that polynomially-tailed losses can exhibit worse
margin behavior than exponentially-tailed losses, but this does not fully
explain why the former are avoided in practice (and in theory).
What are some further consequences on time and sample complexity of these two
loss classes?

Another question is the role of early stopping.
We have established that one can stop a gradient method after a
long-enough training and obtain a predictor with roughly the same direction as a
minimally-regularized predictor.
This, however, requires fairly \emph{late} stopping; what happens for general
losses with aggressively early stopping?
Moreover, could these observations justify the low levels of regularization
encountered in practice?

Lastly, our analysis here does not distinguish the logistic and exponential losses;
meanwhile, the logistic loss (and cross-entropy loss) are dominant in the practice of
classification.   What is a more refined picture for these two losses?  Does it boil
down to the Lipschitz properties of the logistic loss, or is there more?

\acks

ZJ and MT are grateful to the NSF for support under grant IIS-1750051, and to Microsoft
Research for hosting them during various parts of this work.

\bibliography{ab}

\begin{thebibliography}{23}
\providecommand{\natexlab}[1]{#1}
\providecommand{\url}[1]{\texttt{#1}}
\expandafter\ifx\csname urlstyle\endcsname\relax
  \providecommand{\doi}[1]{doi: #1}\else
  \providecommand{\doi}{doi: \begingroup \urlstyle{rm}\Url}\fi

\bibitem[Bartlett(1996)]{bartlett_margin}
Peter~L. Bartlett.
\newblock For valid generalization the size of the weights is more important
  than the size of the network.
\newblock In \emph{NIPS}, 1996.

\bibitem[Bartlett et~al.(2006)Bartlett, Jordan, and
  McAuliffe]{bartlett_jordan_mcauliffe}
Peter~L.\ Bartlett, Michael~I.\ Jordan, and Jon~D.\ McAuliffe.
\newblock Convexity, classification, and risk bounds.
\newblock \emph{Journal of the American Statistical Association}, 101\penalty0
  (473):\penalty0 138--156, 2006.

\bibitem[Bartlett et~al.(2017)Bartlett, Foster, and Telgarsky]{spec}
Peter~L Bartlett, Dylan~J Foster, and Matus~J Telgarsky.
\newblock Spectrally-normalized margin bounds for neural networks.
\newblock In \emph{NIPS}, 2017.

\bibitem[Charles et~al.(2019)Charles, Rajput, Wright, and
  Papailiopoulos]{adversarial_margin}
Zachary Charles, Shashank Rajput, Stephen Wright, and Dimitris Papailiopoulos.
\newblock Convergence and margin of adversarial training on separable data.
\newblock \emph{arXiv preprint arXiv:1905.09209}, 2019.

\bibitem[Freund and Schapire(1997)]{freund_schapire_adaboost}
Yoav Freund and Robert~E. Schapire.
\newblock A decision-theoretic generalization of on-line learning and an
  application to boosting.
\newblock \emph{J. Comput. Syst. Sci.}, 55\penalty0 (1):\penalty0 119--139,
  1997.

\bibitem[Gunasekar et~al.(2018)Gunasekar, Lee, Soudry, and
  Srebro]{soudry_convolutional}
Suriya Gunasekar, Jason~D Lee, Daniel Soudry, and Nati Srebro.
\newblock Implicit bias of gradient descent on linear convolutional networks.
\newblock In \emph{NeurIPS}, pages 9461--9471, 2018.

\bibitem[Ji and Telgarsky(2019{\natexlab{a}})]{deeplinear_alignment}
Ziwei Ji and Matus Telgarsky.
\newblock Gradient descent aligns the layers of deep linear networks.
\newblock In \emph{ICLR}, 2019{\natexlab{a}}.

\bibitem[Ji and Telgarsky(2019{\natexlab{b}})]{riskparam}
Ziwei Ji and Matus Telgarsky.
\newblock Risk and parameter convergence of logistic regression.
\newblock In \emph{COLT}, 2019{\natexlab{b}}.

\bibitem[Li et~al.(2020)Li, Fang, Xu, and Zhao]{li_adv}
Yan Li, Ethan~X Fang, Huan Xu, and Tuo Zhao.
\newblock Implicit bias of gradient descent based adversarial training on
  separable data.
\newblock In \emph{ICLR}, 2020.

\bibitem[Lyu and Li(2020)]{kaifeng_margin}
Kaifeng Lyu and Jian Li.
\newblock Gradient descent maximizes the margin of homogeneous neural networks.
\newblock In \emph{ICLR}, 2020.

\bibitem[Neyshabur et~al.(2014)Neyshabur, Tomioka, and
  Srebro]{behnam__deeplearning_bias}
Behnam Neyshabur, Ryota Tomioka, and Nathan Srebro.
\newblock In search of the real inductive bias: On the role of implicit
  regularization in deep learning.
\newblock \emph{arXiv preprint arXiv:1412.6614}, 2014.

\bibitem[Novikoff(1962)]{novikoff}
Albert~B.J. Novikoff.
\newblock On convergence proofs on perceptrons.
\newblock \emph{In Proceedings of the Symposium on the Mathematical Theory of
  Automata}, 12:\penalty0 615--622, 1962.

\bibitem[Rosset et~al.(2004)Rosset, Zhu, and Hastie]{rosset}
Saharon Rosset, Ji~Zhu, and Trevor Hastie.
\newblock Boosting as a regularized path to a maximum margin classifier.
\newblock \emph{JMLR}, 5:\penalty0 941--973, 2004.

\bibitem[Schapire et~al.(1997)Schapire, Freund, Bartlett, and
  Lee]{boosting_margin}
Robert~E. Schapire, Yoav Freund, Peter Bartlett, and Wee~Sun Lee.
\newblock Boosting the margin: A new explanation for the effectiveness of
  voting methods.
\newblock In \emph{ICML}, pages 322--330, 1997.

\bibitem[Shawe-Taylor et~al.(1998)Shawe-Taylor, Bartlett, Williamson, and
  Anthony]{STBWA1998}
J.~Shawe-Taylor, P.~L. Bartlett, R.~C. Williamson, and M.~Anthony.
\newblock Structural risk minimization over data-dependent hierarchies.
\newblock \emph{IEEE Trans. Inf. Theor.}, 44\penalty0 (5):\penalty0 1926--1940,
  September 1998.

\bibitem[Soudry et~al.(2018)Soudry, Hoffer, Nacson, Gunasekar, and
  Srebro]{soudry_linear}
Daniel Soudry, Elad Hoffer, Mor~Shpigel Nacson, Suriya Gunasekar, and Nathan
  Srebro.
\newblock The implicit bias of gradient descent on separable data.
\newblock \emph{JMLR}, 19\penalty0 (1):\penalty0 2822--2878, 2018.

\bibitem[Telgarsky(2013)]{margins_shrinkage_boosting}
Matus Telgarsky.
\newblock Margins, shrinkage, and boosting.
\newblock In \emph{ICML}, 2013.

\bibitem[Vapnik(1982)]{vapnik}
Vladimir Vapnik.
\newblock \emph{Estimation of Dependences Based on Empirical Data}.
\newblock Springer-Verlag, Berlin, Heidelberg, 1982.

\bibitem[Wei et~al.(2019)Wei, Lee, Liu, and Ma]{wei_regularization_matters}
Colin Wei, Jason~D Lee, Qiang Liu, and Tengyu Ma.
\newblock Regularization matters: Generalization and optimization of neural
  nets vs their induced kernel.
\newblock In \emph{Neurips}, pages 9709--9721, 2019.

\bibitem[Zhang et~al.(2016)Zhang, Bengio, Hardt, Recht, and
  Vinyals]{rethinking}
Chiyuan Zhang, Samy Bengio, Moritz Hardt, Benjamin Recht, and Oriol Vinyals.
\newblock Understanding deep learning requires rethinking generalization.
\newblock \emph{arXiv preprint arXiv:1611.03530}, 2016.

\bibitem[Zhang(2004)]{zhang_convex_consistency}
Tong Zhang.
\newblock Statistical behavior and consistency of classification methods based
  on convex risk minimization.
\newblock \emph{The Annals of Statistics}, 32:\penalty0 56--85, 2004.

\bibitem[Zhang and Yu(2005)]{zhang_yu_boosting}
Tong Zhang and Bin Yu.
\newblock Boosting with early stopping: Convergence and consistency.
\newblock \emph{The Annals of Statistics}, 33:\penalty0 1538--1579, 2005.

\bibitem[Zhao and Yu(2007)]{zhao_yu}
Peng Zhao and Bin Yu.
\newblock Stagewise lasso.
\newblock \emph{JMLR}, 8\penalty0 (Dec):\penalty0 2701--2726, 2007.

\end{thebibliography}

\appendix

\section{Omitted proofs from \Cref{sec:if}}\label{app_sec:if}

\begin{proofof}{\Cref{fact:gd}}
    For any $\barw\in\R^d$, it holds that
    \begin{align}
      \enVert{\vw_{t+1}-\vw}^2 & =\enVert{\vw_t-\vw}^2-2\eta\ip{\nf(\vw_t)}{\vw_t-\vw}+\eta^2\enVert{\nf(\vw_t)}^2 \nonumber \\
       & =\enVert{\vw_t-\vw}^2+2\eta\ip{\nf(\vw_t)}{\vw-\vw_t}+2\eta\cdot \frac{\eta}{2}\enVert{\nf(\vw_t)}^2 \nonumber \\
       & \le\enVert{\vw_t-\vw}^2+2\eta\del{f(\vw)-f(\vw_t)}+2\eta\del{f(\vw_t)-f(\vw_{t+1})} \nonumber \\
       & =\enVert{\vw_t-\vw}^2+2\eta\del{f(\vw)-f(\vw_{t+1})}. \label{eq:gd_sq_dist_proof}
   \end{align}
   On the third line we use the convexity of $f$ and \cref{eq:gd_func}.

   Since $f(\vw_t)$ is nondecreasing, $\lim_{t\to\infty}f(\vw_t)$ exists.
   Suppose $\lim_{t\to\infty}f(\vw_t)>\inf_{\vw\in\R^d}f(\vw)$.
   Let $\barw\in\R^d$ satisfy $f(\barw)<\lim_{t\to\infty}f(\vw_t)-\epsilon$ for
   some $\epsilon>0$.
   It follows from \cref{eq:gd_sq_dist_proof} that
   $\enVert{\vw_{t+1}-\barw}^2\le\enVert{\vw_t-\barw}^2-2\eta\epsilon$ for any $t$,
   which implies $\|\vw_{t+1}-\barw\|^2\to-\infty$, which is a contradiction.
\end{proofof}

\begin{proofof}{\Cref{fact:conv_dir_if}}
    First we show that for any $\epsilon>0$, there exists $B_1(\epsilon)>0$,
    such that for any gradient descent iterate $\vw_t$ with
    $\|\vw_t\|>B_1(\epsilon)$, it holds that
    $\enVert{\nicefrac{\vw_t}{\|\vw_t\|}-\baru}<\epsilon$.
    Given any $\epsilon$, by our assumption, there exists $t_1$ such that
    $\enVert{\nicefrac{\vw_t}{\|\vw_t\|}-\baru}<\epsilon$ for any $t>t_1$.
    It is enough to let $B_1(\epsilon)=\max_{0\le t\le t_1}\|\vw_t\|+1$.

    Then we show that $\lim_{B\to\infty}\ip{\barw(B)}{\baru}\to\infty$.
    If this is not true, then there exists a constant $C>0$ such that there
    exists arbitrarily large $B$ with $\ip{\barw(B)}{\baru}<C$.
    Choose $B_2$ such that
    \begin{align*}
        B_2>\max\cbr{5\del{\|\vw_0\|+C+1},B_1\del{\frac{1}{4}}+1},\quad\textrm{and}\quad\ip{\barw(B_2)}{\baru}<C.
    \end{align*}
    Let $t_2$ denote the first step such that $\|\vw_{t_2}\|>B_2-1$.
    Since $B_2-1>\|\vw_0\|$, we have $t_2>0$.
    Moreover, the conditions of \Cref{fact:conv_dir_if} (i.e., \cref{eq:gd_func}
    and $\eta\le1/\del{2f(\vw_0)}$) implies
    \begin{equation}\label{eq:gd_one_step}
        \begin{split}
            \|\vw_{t_2}-\vw_{t_2-1}\|=\eta\enVert{\nf(\vw_{t_2-1})}= & \sqrt{\eta^2\enVert{\nf(\vw_{t_2-1})}^2} \\
             & \le\sqrt{2\eta\del{f(\vw_{t_2-1})-f(\vw_{t_2})}} \\
             & \le\sqrt{2\eta f(\vw_0)}\le1.
        \end{split}
    \end{equation}
    Therefore from the definition of $t_2$,
    \begin{align*}
        \|\vw_{t_2}\|\le\|\vw_{t_2}-1\|+\|\vw_{t_2}-\vw_{t_2-1}\|\le B_2-1+1=B_2.
    \end{align*}
    By the definition of $t_2$ and $\barw(B_2)$, we have
    $f\del{\barw(B_2)}\le f(\vw_t)$ for any $t\le t_2$.
    Using \cref{eq:gd_sq_dist}, we can show that
    \begin{align}\label{eq:converse_tmp1}
        \enVert{\vw_{t_2}-\barw(B_2)}\le\enVert{\vw_0-\barw(B_2)}.
    \end{align}
    On one hand,
    \begin{align}\label{eq:converse_tmp2}
        \enVert{\vw_0-\barw(B_2)}\le\|\vw_0\|+\enVert{\barw(B_2)}=\|\vw_0\|+B_2.
    \end{align}
    On the other hand,
    \begin{align*}
        \enVert{\vw_{t_2}-\barw(B_2)}^2 & =\|\vw_{t_2}\|^2+B_2^2-2\ip{\vw_{t_2}}{\barw(B_2)} \\
         & =\|\vw_{t_2}\|^2+B_2^2-2\|\vw_{t_2}\|\ip{\frac{\vw_{t_2}}{\|\vw_{t_2}\|}}{\barw(B_2)} \\
         & >(B_2-1)^2+B_2^2-2\|\vw_{t_2}\|\ip{\frac{\vw_{t_2}}{\|\vw_{t_2}\|}}{\barw(B_2)}.
    \end{align*}
    By the definition of $t_2$ and $B_2$, we have
    \begin{align*}
        \|\vw_{t_2}\|>B_2-1>B_1\del{\frac{1}{4}},
    \end{align*}
    and thus $\enVert{\nicefrac{\vw_{t_2}}{\|\vw_{t_2}\|}-\baru}<1/4$.
    As a result,
    \begin{align*}
        \ip{\frac{\vw_{t_2}}{\|\vw_{t_2}\|}}{\barw(B_2)}<\ip{\baru}{\barw(B_2)}+\frac{1}{4}B_2<C+\frac{1}{4}B_2,
    \end{align*}
    and
    \begin{align}
        \enVert{\vw_{t_2}-\barw(B_2)}^2 & >(B_2-1)^2+B_2^2-2\|\vw_{t_2}\|C-\frac{1}{2}\|\vw_{t_2}\|B_2 \nonumber \\
         & \ge(B_2-1)^2+B_2^2-2CB_2-\frac{1}{2}B_2^2>\frac{3}{2}B_2^2-2CB_2-2B_2. \label{eq:converse_tmp3}
    \end{align}
    Combining \cref{eq:converse_tmp1,eq:converse_tmp2,eq:converse_tmp3} gives
    \begin{align*}
        \frac{3}{2}B_2^2-2CB_2-2B_2<\|\vw_0\|^2+2\|\vw_0\|B_2+B_2^2,
    \end{align*}
    which implies
    \begin{align*}
        B_2<4\del{\|\vw_0\|+C+1}+\frac{2\|\vw_0\|^2}{B_2}<4\del{\|\vw_0\|+C+1}+\|\vw_0\|<5\del{\|\vw_0\|+C+1},
    \end{align*}
    a contradiction.

    Next we prove the claim that $\lim_{B\to\infty}\barw(B)/B=\baru$.
    If this is not true, then there exists $\delta>0$, such that there exists
    arbitrarily large $B$ with $\enVert{\nicefrac{\barw(B)}{B}-\baru}>\delta$.
    Choose $B_4$ such that
    \begin{align*}
        \enVert{\frac{\barw(B_4)}{B_4}-\baru}>\delta,\quad\textrm{and}\quad\ip{\barw(B_4)}{\baru}>B_1\del{\frac{\delta^3}{32}}+\|\vw_0\|+1,\quad\textrm{and}\quad B_4>\frac{32}{\delta^3}.
    \end{align*}
    Let $B_3:=\ip{\barw(B_4)}{\baru}$.
    By geometric arguments, we have
    \begin{align}\label{eq:converse_tmp4}
        \enVert{\barw(B_4)-B_4\baru}-\enVert{\barw(B_4)-B_3\baru}>\frac{B_4\delta^3}{8}.
    \end{align}
    Let $t_3$ denote the first step such that $\|\vw_{t_3}\|>B_3-1$.
    Since $B_3-1>\|\vw_0\|$, we have $t_3>0$, and similar to \cref{eq:gd_one_step}
    we can show that $\|\vw_{t_3}\|\le B_3$.
    Since $B_3-1>B_1(\delta^3/32)$, we have
    $\enVert{\nicefrac{\vw_{t_3}}{\|\vw_{t_3}\|}-\baru}<\delta^3/32$.
    As a result,
    \begin{align}\label{eq:converse_tmp5}
        \|\vw_{t_3}-B_3\baru\|\le\enVert{\vw_{t_3}-\|\vw_{t_3}\|\baru}+\enVert{\|\vw_{t_3}\|\baru-B_3\baru}\le\|\vw_{t_3}\|\frac{\delta^3}{32}+1\le \frac{B_3\delta^3}{32}+1\le \frac{B_4\delta^3}{32}+1.
    \end{align}
    Similarly, let $t_4$ denote the first step such that $\|\vw_{t_4}\|>B_4-1$, we
    can show that $\|\vw_{t_4}\|\le B_4$, and
    \begin{align}\label{eq:converse_tmp6}
        \|\vw_{t_4}-B_4\baru\|\le\frac{B_4\delta^3}{32}+1.
    \end{align}
    Combining \cref{eq:converse_tmp4,eq:converse_tmp5,eq:converse_tmp6}
    gives
    \begin{equation}\label{eq:converse_tmp7}
        \begin{split}
            & \quad \enVert{\barw(B_4)-\vw_{t_4}}-\enVert{\barw(B_4)-\vw_{t_3}} \\
            & \ge\enVert{\barw(B_4)-B_4\baru}-\|B_4\baru-\vw_{t_4}\|-\enVert{\barw(B_4)-B_3\baru}-\|B_3\baru-\vw_{t_3}\| \\
             & \ge \frac{B_4\delta^3}{8}-\frac{B_4\delta^3}{32}-1-\frac{B_4\delta^3}{32}-1 \\
             & =\frac{B_4\delta^3}{16}-2>0.
        \end{split}
    \end{equation}
    On the other hand, using \cref{eq:converse_tmp4} and the triangle
    inequality,
    \begin{align*}
        B_4-B_3=\|B_4\baru-B_3\baru\|\ge\enVert{\barw(B_4)-B_4\baru}-\enVert{\barw(B_4)-B_3\baru}>\frac{B_4\delta^3}{8}>4,
    \end{align*}
    and thus $t_4>t_3$.
    Since $\|\vw_{t_4}\|\le B_4$, by the definition of $t_4$ and  $\barw(B_4)$, we
    have $f\del{\barw(B_4)}\le f(\vw_t)$ for any $t\le t_4$.
    Since $t_3<t_4$, \cref{eq:gd_sq_dist} implies
    $\enVert{\barw(B_4)-\vw_{t_4}}\le\enVert{\barw(B_4)-\vw_{t_3}}$, which
    contradicts \cref{eq:converse_tmp7}.
\end{proofof}

\section{Omitted proofs from \Cref{sec:sep}}\label{app_sec:sep}

We first verify that if $\ell$ is $\beta$-smooth, then $\cR$ is also
$\beta$-smooth.
Given $\vw,\vw'\in\R^d$, we have
\begin{align*}
    \enVert{\nR(\vw)-\nR(\vw')} & =\enVert{\frac{1}{n}\sum_{i=1}^{n}\ell'\del{y_i \langle \vw,\vx_i\rangle}y_i\vx_i-\frac{1}{n}\sum_{i=1}^{n}\ell'\del{y_i \langle \vw',\vx_i\rangle}y_i\vx_i} \\
     & \le \frac{1}{n}\sum_{i=1}^{n}\envert{\ell'\del{y_i \langle \vw, \vx_i\rangle}-\ell'\del{y_i \langle \vw',\vx_i\rangle}}\|y_i\vx_i\| \\
     & \le \frac{1}{n}\sum_{i=1}^{n}\beta\envert{y_i \langle \vw,\vx_i\rangle-y_i \langle \vw',\vx_i\rangle} \\
     & \le \frac{1}{n}\sum_{i=1}^{n}\beta\|\vw-\vw'\|\|y_i\vx_i\|\le \beta\|\vw-\vw'\|.
\end{align*}
Therefore $\cR$ is $\beta$-smooth.

To proceed, we first need the following lemma.
\begin{lemma}\label[lemma]{fact:barw_collinear}
    It holds that
    \begin{align*}
        \frac{\barw(B)}{B}=-\frac{\nR\del{\barw(B)}}{\enVert{\nR\del{\barw(B)}}}.
    \end{align*}
    Conversely, if $\|\vw\|=B$ and $\vw/B=-\nR(\vw)/\enVert{\nR(\vw)}$, then
    $\vw=\barw(B)$.
\end{lemma}
\begin{proof}
    By the first order optimality conditions, $\vw=\barw(B)$ if and only if for any
    $\vw'$ with $\|\vw'\|_2\le B$, it holds that
    \begin{align}\label{eq:1st}
        \ip{\nR(\vw)}{\vw'-\vw}\ge0.
    \end{align}
    Since the infimum of $\cR$ is not attained, the gradient $\nR(\vw)$ is always
    nonzero.
    The structure of the $\ell_2$ ball implies that \cref{eq:1st} holds if and
    only if $\|\vw\|=B$ and $\vw/B=-\nR(\vw)/\enVert{\nR(\vw)}$.
\end{proof}

\begin{proofof}{\Cref{fact:baru_margin}}
    Since $\barw(B)/B\to\baru$, the margin of $\barw(B)/B$ converges to the
    margin of $\baru$.
    For large enough $B$, the risk $\cR\del{\barw(B)}\le\ell(0)/n$, which
    implies $\barw(B)/B$ has a nonnegative margin, and thus $\baru$ also has a
    nonnegative margin.

    The proof of \Cref{fact:baru_margin} is by contradiction.
    Given $\epsilon:=\hgamma^2/(2n)$, suppose there exists $B_0>0$, such that
    for any $B\ge B_0$, the margin of $\barw(B)/B$ is no larger than $\epsilon$.
    We will derive a contradiction, which implies that the margin of $\baru$ is
    at least $\hgamma^2/(2n)$.

    For any $B>0$, \Cref{fact:barw_collinear} ensures that
    \begin{align}\label{eq:align}
        -\ip{\frac{\barw(B)}{B}}{\nR\del{\barw(B)}}=\enVert{\nR\del{\barw(B)}}.
    \end{align}
    For simplicity, let $\vz_i:=y_i\vx_i$.
    The left hand side of \cref{eq:align} can be rewritten as
    \begin{align}\label{eq:align_lhs}
        \frac{1}{n}\sum_{i=1}^{n}-\ell'\del{\ip{\barw(B)}{\vz_i}}\ip{\frac{\barw(B)}{B}}{\vz_i},
    \end{align}
    while the right hand side of \cref{eq:align} can be bounded below as
    \begin{align}\label{eq:align_rhs}
        \enVert{\nR\del{\barw(B)}}\ge\ip{-\nR\del{\barw(B)}}{\hu}
        &\ge \frac{1}{n}\sum_{i=1}^{n}-\ell'\del{\ip{\barw(B)}{\vz_i}}\hgamma,
    \end{align}
    where $\hu$ denotes the unit maximum margin predictor.
    Let $H$ denote the set of data points on which $\barw(B)/B$ has margin
    larger than $\hgamma$, and suppose without loss of generality
    that $\barw(B)/B$ achieves its minimum margin on
    $\vz_1$.
    It follows from \cref{eq:align,eq:align_lhs,eq:align_rhs} that
    \begin{align}
        \sum_{\vz_i\in H}^{}-\ell'\del{\ip{\barw(B)}{\vz_i}}\del{\ip{\frac{\barw(B)}{B}}{\vz_i}-\hgamma} & \ge \sum_{\vz_i\not\in H}^{}-\ell'\del{\ip{\barw(B)}{\vz_i}}\del{\hgamma-\ip{\frac{\barw(B)}{B}}{\vz_i}} \nonumber \\
        & \ge -\ell'\del{\ip{\barw(B)}{\vz_1}}\del{\hgamma-\ip{\frac{\barw(B)}{B}}{\vz_1}}. \label{eq:align_sim}
    \end{align}

    Now consider $B\ge B_0$, which implies $\ip{\barw(B)/B}{\vz_1}\le\epsilon$.
    Since $\epsilon<\hgamma/2$, and $\|\vz_i\|_2\le1$, \cref{eq:align_sim} implies
    \begin{align*}
        -n\ell'(B\hgamma)\ge-\ell'(B\epsilon)(\hgamma-\epsilon)\ge-\ell'(B\epsilon)\frac{\hgamma}{2},
    \end{align*}
    and thus
    \begin{align}\label{eq:ratio}
        \frac{-\ell'(B\epsilon)}{-\ell'(B\hgamma)}\le\frac{2n}{\hgamma}
    \end{align}
    for all $B\ge B_0$.
    Let $\alpha:=B_0\epsilon=B_0\hgamma^2/(2n)$, and $\lambda:=2n/\hgamma$.
    For any $k\ge1$, we have
    \begin{align*}
        \int_{\alpha\lambda^k}^{\alpha\lambda^{k+1}}-\ell'(z)\dif z=\int_{\alpha\lambda^{k-1}}^{\alpha\lambda^k}-\ell'(\lambda y)\lambda\dif y\ge\int_{\alpha\lambda^{k-1}}^{\alpha\lambda^k}-\ell'(y)\dif y,
    \end{align*}
    where \cref{eq:ratio} is used.
    By induction, we have
    \begin{align*}
        \int_{\alpha\lambda^k}^{\alpha\lambda^{k+1}}-\ell'(z)\dif z\ge \int_{\alpha}^{\alpha\lambda}-\ell'(z)\dif z>0.
    \end{align*}
    As a result,
    \begin{align*}
        \int_\alpha^\infty-\ell'(z)\dif z=\infty,
    \end{align*}
    which is contradiction, since
    $\int_\alpha^\infty-\ell'(z)\dif z=\ell(\alpha)$ should be finite.
\end{proofof}

When the loss function has a polynomial tail $az^{-b}$, then we can use
\cref{eq:ratio} to prove a margin lower bound of
$\hgamma^{(b+2)/(b+1)}n^{-1/(b+1)}$.
The dependency on $n$ cannot be improved in general (cf.
\Cref{fact:baru_margin_ub}).

\begin{proofof}{\Cref{fact:baru_dec}}
    Since $\lim_{B\to\infty}\barw(B)/B=\baru$, we can choose $\rho(\alpha)$
    large enough such that for any $\vw$ with $\|\vw\|>\rho(\alpha)$, it holds that
    \begin{align*}
        \enVert{\barw\del{\|\vw\|}/\|\vw\|-\baru}\le\alpha\bar{\gamma}.
    \end{align*}
    In this case, for any $1\le i\le n$,
    \begin{align*}
        y_i\ip{\barw\del{\|\vw\|}}{\vx_i} & =y_i\ip{\barw\del{\|\vw\|}-\|\vw\|\baru}{\vx_i}+y_i\ip{\|\vw\|\baru}{\vx_i} \\
         & \le \alpha\bgamma\|\vw\|+y_i\ip{\|\vw\|\baru}{\vx_i} \\
         & \le y_i\ip{(1+\alpha)\|\vw\|\baru}{\vx_i}.
    \end{align*}
    As a result,
    \begin{align*}
        \cR\del{(1+\alpha)\enVert{\vw}\baru}\le\cR\del{\barw\del{\|\vw\|}}\le\cR(\vw).
    \end{align*}
\end{proofof}

Next we prove the ``only if'' part of \Cref{fact:conv_dir}.
\begin{proofof}{\Cref{fact:conv_dir}, the ``only if'' part}
    Given any $\epsilon\in(0,1)$, let $\alpha$ satisfy
    $1/(1+\alpha)=1-\epsilon$ (i.e., let $\alpha=\epsilon/(1-\epsilon)$).
    Since $\lim_{t\to\infty}\|\vw_t\|=\infty$, we can choose a step $t_0$ such
    that for any $t\ge t_0$, it holds that $\|\vw_t\|>\max\cbr{\rho(\alpha),1}$,
    where $\rho$ is given by \Cref{fact:baru_dec}.

    Now for any $t\ge t_0$, using the convexity of $\cR$ and
    \Cref{fact:baru_dec}, we have
    \begin{align*}
        \ip{\nR(\vw_t)}{\vw_t-(1+\alpha)\|\vw_t\|\baru}\ge\cR(\vw_t)-\cR\del{(1+\alpha)\|\vw_t\|\baru}\ge0,
    \end{align*}
    meaning
    \begin{align*}
        \ip{\nR(\vw_t)}{\vw_t}\ge(1+\alpha)\|\vw_t\|\ip{\nR(\vw_t)}{\baru}.
    \end{align*}
    Consequently,
    \begin{align*}
      \langle \vw_{t+1}-\vw_t,\baru\rangle & =\ip{-\eta\nR(\vw_t)}{\baru} \\
       & \ge\ip{-\eta\nR(\vw_t)}{\vw_t}\frac{1}{(1+\alpha)\|\vw_t\|} \\
       & =\langle \vw_{t+1}-\vw_t,\vw_t\rangle \frac{1}{(1+\alpha)\|\vw_t\|} \\
       & =\del{\frac{1}{2}\|\vw_{t+1}\|^2-\frac{1}{2}\|\vw_t\|^2-\frac{1}{2}\|\vw_{t+1}-\vw_t\|^2}\frac{1}{(1+\alpha)\|\vw_t\|}.
    \end{align*}
    On one hand, we have
    \begin{align*}
      \del{\frac{1}{2}\|\vw_{t+1}\|^2-\frac{1}{2}\|\vw_t\|^2}/\|\vw_t\|\ge\|\vw_{t+1}\|-\|\vw_t\|.
    \end{align*}
    On the other hand, using the step size condition in \cref{eq:gd_risk},
    we have
    \begin{align*}
      \frac{\|\vw_{t+1}-\vw_t\|^2}{2(1+\alpha)\|\vw_t\|}\le \frac{\|\vw_{t+1}-\vw_t\|^2}{2}=\frac{\eta^2\enVert{\nR(\vw_t)}^2}{2}\le\eta\del{\cR(\vw_t)-\cR(\vw_{t+1})}.
    \end{align*}
    As a result,
    \begin{align*}
      \langle \vw_t-\vw_{t_0},\baru\rangle\ge \frac{\|\vw_t\|-\|\vw_{t_0}\|}{1+\alpha}-\eta\cR(\vw_{t_0})=\del{1-\epsilon}\del{\|\vw_t\|-\|\vw_{t_0}\|}-\eta\cR(\vw_{t_0}),
    \end{align*}
    meaning
    \begin{align*}
        \ip{\frac{\vw_t}{\|\vw_t\|}}{\baru}\ge1-\epsilon+\frac{\langle \vw_{t_0},\baru\rangle-(1-\epsilon)\|\vw_{t_0}\|-\eta\cR(\vw_{t_0})}{\|\vw_t\|}.
    \end{align*}
    Consequently,
    \begin{align*}
      \liminf_{t\to\infty}\ip{\frac{\vw_t}{\|\vw_t\|}}{\baru}\ge1-\epsilon.
    \end{align*}
    Since $\epsilon$ is arbitrary, we get $\vw_t/\|\vw_t\|\to\baru$.
\end{proofof}

\subsection{Omitted proofs from \Cref{sec:sep_dir}}\label{app_sec:sep_dir}

\begin{proofof}{\Cref{fact:exp_baru}}
    First let us verify that the maximum-margin solution $\hu$ is unique.
    If this is not true, suppose there exist two unit vectors $\vu_1$ and $\vu_2$
    which both attain the maximum margin $\hgamma$ but $\vu_1\ne \vu_2$.
    Consider $\vu_3=(\vu_1+\vu_2)/2$.
    Then for any $i$, it holds that
    \begin{align*}
        y_i \langle \vu_3,\vx_i\rangle=y_i \langle \vu_1,\vx_i\rangle/2+y_i \langle \vu_2,\vx_i\rangle/2\ge\hgamma,
    \end{align*}
    and thus $\vu_3$ also maximizes the margin.
    However, since $\vu_1\ne\vu_2$, it follows that $\|\vu_3\|\le1$.
    Consequently, the unit vector $\vu_3/\|\vu_3\|$ should achieve a margin larger
    than $\hgamma$, which is a contradiction.

    Now note that
    \begin{align*}
        \cR(B\hu)=\frac{1}{n}\sum_{i=1}^{n}\ell\del{y_i\langle B\hu,\vx_i\rangle}\le\ell(B\hgamma).
    \end{align*}
    When $B$ is large enough, we have
    \begin{align*}
        \cR(B\hu)\le\ell(B\hgamma)\le2a\exp(-bB\hgamma).
    \end{align*}

    Now suppose \Cref{fact:exp_baru} is not true.
    Then there exists $\epsilon>0$, such that there exists arbitrarily large $B$
    with $\enVert{\nicefrac{\bar{\vw}(B)}{B}-\hu}>\epsilon$.
    Since $\hu$ is the unique maximum-margin solution, it follows that there
    exists $\epsilon'\in(0,\hgamma)$ such that
    \begin{align*}
        \min_{1\le i\le n}y_i \ip{\frac{\barw(B)}{B}}{\vx_i}\le\hgamma-\epsilon',
    \end{align*}
    and thus
    \begin{align*}
        \cR\del{\barw(B)}\ge \frac{1}{n}\ell\del{B(\hgamma-\epsilon')}.
    \end{align*}
    For large enough $B$, it follows that
    \begin{align*}
        \cR\del{\barw(B)}\ge \frac{1}{n}\ell\del{B(\hgamma-\epsilon')}\ge \frac{a}{2n}\exp\del{-bB(\hgamma-\epsilon')}=a\exp\del{-bB\hgamma}\frac{\exp(bB\epsilon')}{2n}.
    \end{align*}
    Since $B$ can be arbitrarily large, the factor $\exp(bB\epsilon')/(2n)$ can
    also be arbitrarily large, which would give $\cR\del{\barw(B)}>\cR(B\hu)$, a
    contradiction.
\end{proofof}

\begin{proofof}{\Cref{fact:poly_baru}}
    The fundamental theorem of calculus implies
    $\ell(z)=\int_z^\infty-\ell'(z)\dif z$, and thus $b>1$.
    We consider the loss function
    \begin{equation*}
        \tilde{\ell}(z):=
        \begin{dcases}
            \frac{a}{b-1}z^{-b+1}, & \textrm{if }z\ge1, \\
            -az+\frac{ab}{b-1}, & \textrm{if }z<1.
        \end{dcases}
    \end{equation*}
    It can be verified that $\tilde{\ell}$ is convex, differentiable, and
    strictly decreasing to $0$.
    Moreover, we have $-\tilde{\ell}'(z)=az^{-b}$ for $z\ge1$.

    Let $\widetilde{\cR}$ denote the empirical risk function using loss
    $\tilde{\ell}$.
    Let $B_0$ be large enough such that
    \begin{align*}
        \min_{\vw:\|\vw\|_2\le B_0}\widetilde{\cR}(\vw)<\frac{1}{n}\tilde{\ell}(1)=\frac{a}{n(b-1)},
    \end{align*}
    and let $\baru$ denote the direction of the optimal solution:
    \begin{align*}
        \argmin_{\vw:\|\vw\|_2\le B_0}\widetilde{\cR}(\vw)=B_0\baru.
    \end{align*}
    Due to \Cref{fact:barw_collinear}, we have
    \begin{align*}
        \baru=-\frac{\nabla\widetilde{\cR}(B_0\baru)}{\enVert{\nabla\widetilde{\cR}(B_0\baru)}}=-\frac{1}{\enVert{\nabla\widetilde{\cR}(B_0\baru)}}\frac{1}{n}\sum_{i=1}^{n}\tilde{\ell}'\del{y_i\langle B_0\baru,\vx_i\rangle}y_i\vx_i.
    \end{align*}
    Since $\widetilde{\cR}(B_0\baru)<\tilde{\ell}(1)/n$, it follows that
    $y_i\langle B_0\baru,\vx_i\rangle>1$ for all $i$, and thus
    \begin{align*}
        \baru=-\frac{1}{\enVert{\nabla\widetilde{\cR}(B_0\baru)}}\frac{1}{n}\sum_{i=1}^{n}\tilde{\ell}'\del{y_i\langle B_0\baru,\vx_i\rangle}y_i\vx_i=\frac{1}{\enVert{\nabla\widetilde{\cR}(B_0\baru)}}\frac{1}{n}\sum_{i=1}^{n}a\del{y_i\langle B_0\baru,\vx_i\rangle}^{-b}y_i\vx_i.
    \end{align*}
    The direction of the right hand side does not depend on $B_0$ due to the
    polynomial tail, and thus for any $B>B_0$, we have
    \begin{align*}
        \baru=-\frac{\nabla\widetilde{\cR}(B\baru)}{\enVert{\nabla\widetilde{\cR}(B\baru)}},
    \end{align*}
    and thus \Cref{fact:barw_collinear} ensures
    \begin{align*}
        \argmin_{\vw:\|\vw\|_2\le B}\widetilde{\cR}(\vw)=B\baru.
    \end{align*}

    Now we consider the original loss $\ell$.
    We claim that $\lim_{B\to\infty}\barw(B)/B\to\baru$.
    First note that $\nR\del{\barw(B)}/\enVert{\nR\del{\barw(B)}}$ and        $\nabla\widetilde{\cR}\del{\barw(B)}/\enVert{\nabla\widetilde{\cR}\del{\barw(B)}}$
    can become arbitrarily close as $B\to\infty$.
    To see this, define
    \begin{align*}
        q_i(B):=\frac{\ell'\del{y_i \ip{\barw(B)}{\vx_i}}}{\sum_{j=1}^{n}\ell'\del{y_j \ip{\barw(B)}{\vx_j}}},\quad\mathrm{and}\quad \tilde{q}_i(B):=\frac{\tilde{\ell}'\del{y_i \ip{\barw(B)}{\vx_i}}}{\sum_{j=1}^{n}\tilde{\ell}'\del{y_j \ip{\barw(B)}{\vx_j}}}.
    \end{align*}
    Note that
    \begin{align*}
        -\frac{\nR\del{\barw(B)}}{\enVert{\nR\del{\barw(B)}}}=\frac{\sum_{i=1}^{n}q_i(B)y_i\vx_i}{\enVert{\sum_{i=1}^{n}q_i(B)y_i\vx_i}},\quad\mathrm{and}\quad-\frac{\nabla\widetilde{\cR}\del{\barw(B)}}{\enVert{\nabla\widetilde{\cR}\del{\barw(B)}}}=\frac{\sum_{i=1}^{n}\tilde{q}_i(B)y_i\vx_i}{\enVert{\sum_{i=1}^{n}\tilde{q}_i(B)y_i\vx_i}}.
    \end{align*}
    By the conditions of \Cref{fact:poly_baru}, it holds that
    $\envert{q_i(B)-\tilde{q}_i(B)}\to0$ for all $1\le i\le n$ as $B\to\infty$,
    and thus
    \begin{align*}
        \envert{\enVert{\sum_{i=1}^{n}q_i(B)y_i\vx_i}-\enVert{\sum_{i=1}^{n}\tilde{q}_i(B)y_i\vx_i}}\to0.
    \end{align*}
    Moreover, for any $q\in\Delta_n$ (i.e., $q_i\ge0$ and
    $\sum_{i=1}^{n}q_i=1$), it holds that
    \begin{align*}
        \enVert{\sum_{i=1}^{n}q_iy_i\vx_i}\ge\ip{\sum_{i=1}^{n}q_iy_i\vx_i}{\hu}\ge\hgamma>0,
    \end{align*}
    where $\hu$ and $\hgamma$ denote the maximum-margin solution and the maximum
    margin.
    Consequently $\nR\del{\barw(B)}/\enVert{\nR\del{\barw(B)}}$ and        $\nabla\widetilde{\cR}\del{\barw(B)}/\enVert{\nabla\widetilde{\cR}\del{\barw(B)}}$
    can become arbitrarily close.
    By \Cref{fact:barw_collinear},
    \begin{align*}
        \frac{\barw(B)}{B}=-\frac{\nR\del{\barw(B)}}{\enVert{\nR\del{\barw(B)}}},
    \end{align*}
    and thus $\barw(B)/B$ and
    $-\nabla\widetilde{\cR}\del{\barw(B)}/\enVert{\nabla\widetilde{\cR}\del{\barw(B)}}$
    can also become arbitrarily close.

    Suppose $\barw(B)/B$ does not converge to $\baru$.
    Then there exists $\epsilon>0$ such that there exists arbitrarily large $B$
    with $\enVert{\nicefrac{\barw(B)}{B}-\baru}>\epsilon$.
    When $B$ is large enough, $\barw(B)$ and
    $\nabla\widetilde{\cR}\del{\barw(B)}$
    can be arbitrarily close to collinear, and due to the structure of the
    $\ell_2$ ball, we have
    \begin{align*}
        \ip{\nabla\widetilde{\cR}\del{\barw(B)}}{B\baru-\barw(B)}>0,
    \end{align*}
    which implies that $\widetilde{\cR}(B\baru)>\widetilde{\cR}\del{\barw(B)}$,
    a contradiction.
\end{proofof}

\begin{proofof}{\Cref{fact:baru_margin_ub}}
    Consider the training set $\{(\vx_i,y_i)\}_{i=1}^n$ where $\vx_i=(0.1,0.1)$ for
    $1\le i\le n-1$ and $\vx_n=(0.6,-0.8)$, and $y_i=+1$ for all $1\le i\le n$.
    Note that as we increase $n$, the maximum margin does not change, and thus
    is a universal constant.
    Further consider a loss function such that $\ell(z)=z^{-b}$ for $b>0$ and
    $z\ge1$.
    We will show that the limiting direction $\baru$ induced by $\ell$ satisfies
    \begin{align*}
        y_n\langle\baru,\vx_n\rangle=\Theta\del{\frac{1}{n^{1/(b+1)}}}.
    \end{align*}
    Consequently, for large enough $n$ it holds that $\baru\ne\hu$.

    The existence of $\baru$ is ensured by \Cref{fact:poly_baru}.
    Let $\baru=(u_1,u_2)$, and
    \begin{align*}
        p:=\frac{1}{(0.1u_1+0.1u_2)^{b+1}},\quad\textrm{and}\quad q=\frac{1}{(0.6u_1-0.8u_2)^{b+1}}.
    \end{align*}
    It follows from the proof of \Cref{fact:poly_baru} that $p>0$, $q>0$, and
    $(u_1,u_2)$ is collinear with $\del{0.1(n-1)p+0.6q,0.1(n-1)p-0.8q}$.
    Note that we always have $u_1>0$, and when $n$ is large enough, we also have
    $u_2>0$.
    Consequently,
    \begin{align*}
        \frac{u_1}{u_2}=\frac{0.1(n-1)p+0.6q}{0.1(n-1)p-0.8q}=\frac{(n-1)p/q+6}{(n-1)p/q-8},
    \end{align*}
    and thus
    \begin{align*}
        \frac{p}{q}=\frac{1}{n-1}\frac{8u_1+6u_2}{u_1-u_2}.
    \end{align*}
    Since $0.6u_1-0.8u_2>0$ and $u_1^2+u_2^2=1$, it can be shown that
    $u_1-u_2>0.2$, and thus $p/q=\Theta(1/n)$.
    Moreover,
    \begin{align*}
        \frac{p}{q}=\frac{(0.6u_1-0.8u_2)^{b+1}}{(0.1u_1+0.1u_2)^{b+1}},
    \end{align*}
    and thus
    \begin{align*}
        y_n\langle \baru,\vx_n\rangle=0.6u_1-0.8u_2=\Theta\del{\frac{1}{n^{1/(b+1)}}}.
    \end{align*}
\end{proofof}

To prove \Cref{fact:no_baru}, we first need the following result which allows us
to switch between different tails.
\begin{lemma}\label[lemma]{fact:switch}
    Consider the loss functions $\lexp(z):=e^{-z}$ and $\lrec(z):=1/z$ on
    $[1,\infty)$.
    Given any $C_0>1$, there exists $C_1>C_0$ and a convex loss $\ell_1$ such
    that $\ell_1=\lexp$ on $[1,C_0]$, and $\ell_1=\lrec$ on $[C_1,\infty)$, and
    $\ell_1$ is $2$-smooth.
    Similarly, there also exists $C_2>C_0$ and convex loss $\ell_2$ such that
    $\ell_2=\lrec$ on $[1,C_0]$, and $\ell_2=\lexp$ on $[C_2,\infty)$, and
    $\ell_2$ is $2$-smooth.
\end{lemma}
\begin{proofof}{\Cref{fact:switch}}
    Let $C_1$ be large enough such that
    \begin{align}\label{eq:c1_cond}
        \frac{1}{C_1}+\frac{1}{C_1^2}(C_1-C_0)+\frac{1}{2}e^{-C_0}-\frac{1}{2C_1^2}<e^{-C_0},\quad\textrm{and}\quad C_1>C_0+\frac{3}{2}.
    \end{align}
    Consider the two lines
    \begin{align*}
        f_1(z):=e^{-C_0}-e^{-C_0}(z-C_0),\quad\textrm{and}\quad f_2(z):=\frac{1}{C_1}-\frac{1}{C_1^2}(z-C_1)+\frac{1}{2}e^{-C_0}-\frac{1}{2C_1^2}.
    \end{align*}
    Note that due to \cref{eq:c1_cond}, we have
    \begin{align*}
        f_1(C_0)=e^{-C_0}>\frac{1}{C_1}+\frac{1}{C_1^2}(C_1-C_0)+\frac{1}{2}e^{-C_0}-\frac{1}{2C_1^2}=f_2(C_0),
    \end{align*}
    and
    \begin{align*}
        f_1(C_1-1) & =e^{-C_0}-e^{-C_0}(C_1-1-C_0) \\
         & <e^{-C_0}-\frac{1}{2}e^{-C_0} \\
         & <\frac{1}{2}e^{-C_0}+\frac{1}{C_1}+\frac{1}{2C_1^2}=f_2(C_1-1).
    \end{align*}
    Consequently, the two lines $f_1$ and $f_2$ intersect at some point
    $C\in(C_0,C_1-1)$.
    Now we define
    \begin{equation*}
        \ell_1'(z)=
        \begin{dcases}
            -e^{-C_0}, & \textrm{if }z\in[C_0,C], \\
            -e^{-C_0}+\del{e^{-C_0}-\frac{1}{C_1^2}}(z-C), & \textrm{if }z\in[C,C+1], \\
            -\frac{1}{C_1^2}, & \textrm{if }z\in[C+1,C_1].
        \end{dcases}
    \end{equation*}
    It is easy to verify that $\ell'$ is nondecreasing $2$-Lipschitz on
    $[C_0,C_1]$.
    We only need to show that
    \begin{align}\label{eq:loss_full}
        \int_{C_0}^{C_1}\ell_1'(z)\dif z=\frac{1}{C_1}-e^{-C_0}.
    \end{align}
    Note that
    \begin{align}\label{eq:loss_p1}
        \int_{C_0}^C\ell_1'(z)\dif z=-e^{-C_0}(C-C_0),
    \end{align}
    and
    \begin{align}\label{eq:loss_p2}
        \int_C^{C+1}\ell_1'(z)\dif z=-\frac{1}{2}e^{-C_0}-\frac{1}{2C_1^2},
    \end{align}
    and
    \begin{align}\label{eq:loss_p3}
        \int_{C+1}^{C_1}\ell_1'(z)\dif z=-\frac{1}{C_1^2}(C_1-C-1).
    \end{align}
    Moreover, since $f_1$ and $f_2$ intersect at $C$, we have
    \begin{align}\label{eq:intersect}
        e^{-C_0}-e^{-C_0}(C-C_0)=\frac{1}{C_1}-\frac{1}{C_1^2}(C-C_1)+\frac{1}{2}e^{-C_0}-\frac{1}{2C_1^2}.
    \end{align}
    Combining \cref{eq:loss_p1,eq:loss_p2,eq:loss_p3,eq:intersect} proves
    \cref{eq:loss_full}.

    The proof of the other claim is similar.
    Let $C_2$ be large enough such that
    \begin{align*}
        \frac{1}{C_0}>e^{-C_2}-e^{-C_2}(C_0-C_2)+\frac{1}{2C_0^2}-\frac{1}{2}e^{-C_2},\quad\textrm{and}\quad C_2>2C_0+1.
    \end{align*}
    Consider the two lines
    \begin{align*}
        g_1(z):=\frac{1}{C_0}-\frac{1}{C_0^2}(z-C_0),\quad\textrm{and}\quad g_2(z):=e^{-C_2}-e^{-C_2}(z-C_2)+\frac{1}{2C_0^2}-\frac{1}{2}e^{-C_2}.
    \end{align*}
    It can be verified that $g_1(C_0)>g_2(C_0)$ and $g_1(C_2-1)<g_2(C_2-1)$, and
    thus $g_1$ and $g_2$ intersect at some point $C'\in(C_0,C_2-1)$.
    Let
    \begin{equation*}
        \ell_2'(z)=
        \begin{dcases}
            -\frac{1}{C_0^2}, & \textrm{if }z\in[C_0,C'], \\
            -\frac{1}{C_0^2}+\del{\frac{1}{C_0^2}-e^{-C_2}}(z-C), & \textrm{if }z\in[C',C'+1], \\
            -e^{-C_2}, & \textrm{if }z\in[C+1,C_2].
        \end{dcases}
    \end{equation*}
    It can be verified similarly that
    \begin{align*}
        \int_{C_0}^{C_2}\ell_2'(z)\dif z=e^{-C_2}-\frac{1}{C_0}.
    \end{align*}
\end{proofof}

Next we prove \Cref{fact:no_baru}.
We make $\ell$ keep switching between $e^{-z}$ and $1/z$ so that the
regularization path does not converge.
\begin{proofof}{\Cref{fact:no_baru}}
    In this proof, the notation $\barw_{\ell}(B)$ means the regularized solution
    using loss $\ell$.

    Consider the dataset given in the proof of \Cref{fact:baru_margin_ub}.
    If $n$ is large enough, then for the exponential loss $e^{-z}$ we have
    \begin{align*}
        \lim_{B\to\infty}\frac{\barw_{\exp}(B)}{B}=\hu,
    \end{align*}
    while for the reciprocal loss $1/z$ it holds that
    \begin{align*}
        \lim_{B\to\infty}\frac{\barw_{\textrm{recip}}(B)}{B}=\baru\ne\hu.
    \end{align*}
    Let $B_0$ be large enough such that for any $B\ge B_0$,
    \begin{align*}
        \enVert{\frac{\barw_{\exp}(B)}{B}-\hu}\le \frac{\|\baru-\hu\|}{3},\quad\textrm{and}\quad\enVert{\frac{\barw_{\textrm{recip}}(B)}{B}-\baru}\le \frac{\|\baru-\hu\|}{3},
    \end{align*}
    and the margin of $\barw_{\exp}(B)/B$ is at least $\hgamma/2$, and the
    margin of $\barw_{\textrm{recip}}(B)/B$ is at least $\bgamma/2$.

    We construct $\ell$ in the following way.
    Let $\ell(z):=z^2-z+1$ for $z<0$, and $\ell(z):=e^{-z}$ for $z\in[0,B_0]$.
    One can verify that $\ell$ is convex and $1$-smooth on $(-\infty,B_0]$.
    Let $a_0=0$, $b_0=B_0$.
    Now for any $k\ge1$, the construction is as follows.
    \begin{enumerate}
        \item Given $\ell=e^{-z}$ on $[a_{k-1},b_{k-1}]$, \Cref{fact:switch}
        ensures that we can switch $\ell$ to $1/z$: there exists $c_k>b_{k-1}$
        such that we can let $\ell(z)=1/z$ for any $z\ge c_k$.
        We let $\ell(z)=1/z$ on $[c_k,d_k]$ where $d_k:=2nc_k/\bgamma$.
        With this construction it holds that
        $\barw_{\ell}(d_k)=\barw_{\textrm{recip}}(d_k)$.
        To see this, first note that by our condition
        \begin{align*}
            y_i\ip{\barw_{\textrm{recip}}(d_k)}{\vx_i}\ge \frac{d_k\bgamma}{2}=nc_k,\quad\textrm{and}\quad y_i\ip{\barw_{\textrm{recip}}(d_k)}{\vx_i}\le\enVert{\barw_{\textrm{recip}}(d_k)}\|\vx_i\|=d_k
        \end{align*}
        for all $1\le i\le n$, which implies
        \begin{align*}
            \cR_{\ell}\del{\barw_{\textrm{recip}}(d_k)}\le \frac{1}{nc_k}.
        \end{align*}
        On the other hand, if $\barw_{\ell}(d_k)\ne\barw_{\textrm{recip}}(d_k)$,
        then we must have
        \begin{align*}
            y_i\ip{\barw_{\ell}(d_k)}{\vx_i}<c_k
        \end{align*}
        for some $(\vx_i,y_i)$, and it follows that
        \begin{align*}
            \cR_{\ell}\del{\barw_{\ell}(d_k)}>\frac{1}{n}\ell(c_k)=\frac{1}{nc_k}\ge\cR_{\ell}\del{\barw_{\textrm{recip}}(d_k)},
        \end{align*}
        a contradiction.

        \item Given $\ell=1/z$ on $[c_k,d_k]$, \Cref{fact:switch} ensures that
        we can switch $\ell$ to $e^{-z}$: there exists $a_k>d_k$
        such that we can let $\ell(z)=e^{-z}$ for any $z\ge a_k$.
        We let $\ell(z)=e^{-z}$ on $[a_k,b_k]$ where
        $b_k=2\del{a_k+\ln(n)}/\bgamma$.
        Similarly we can show that $\barw_{\ell}(b_k)=\barw_{\exp}(b_k)$.
    \end{enumerate}
    Since for any $B\ge B_0$, it holds that
    \begin{align*}
        \enVert{\frac{\barw_{\exp}(B)}{B}-\frac{\barw_{\textrm{recip}}(B)}{B}}\ge \frac{\|\baru-\hu\|}{3},
    \end{align*}
    the loss $\ell$ constructed above satisfies the requirements in
    \Cref{fact:no_baru}.
\end{proofof}

\section{Omitted proofs from \Cref{sec:non_sep}}\label{app_sec:non_sep}

\begin{proofof}{\Cref{fact:baru_margin_gen}}
    \Cref{fact:barw_collinear} ensures that $\barw(B)$ and $\nR\del{\barw(B)}$
    are collinear, which also implies $\barw_\perp(B):=\Pip\barw(B)$ and
    $\Pip\nR\del{\barw(B)}$ are collinear.
    Formally,
    \begin{align}\label{eq:align_perp}
        -\ip{\frac{\barw_\perp(B)}{\enVert{\barw_\perp(B)}}}{\Pip\nR\del{\barw(B)}}=\enVert{\Pip\nR\del{\barw(B)}},
    \end{align}
    and the left hand side is equal to
    \begin{align}\label{eq:align_perp_lhs}
        \frac{1}{n}\sum_{i=1}^{n}-\ell'\del{\ip{\barw(B)}{y_i\vx_i}}\ip{\frac{\barw_\perp(B)}{\enVert{\barw_\perp(B)}}}{\Pip y_i\vx_i}=\frac{1}{n}\sum_{(\vx_i,y_i)\in D_c}^{}-\ell'\del{\ip{\barw(B)}{y_i\vx_i}}\ip{\frac{\barw_\perp(B)}{\enVert{\barw_\perp(B)}}}{y_i\vx_i}.
    \end{align}
    Let
    \begin{align*}
        \hu:=\argmax_{\|\vu\|=1,\vu\in S^\perp}\min_{(\vx_i,y_i)\in D_c}y_i \langle \vu,\vx_i\rangle,\quad\textrm{and}\quad\hgamma:=\max_{\|\vu\|=1,\vu\in S^\perp}\min_{(\vx_i,y_i)\in D_c}y_i \langle \vu,\vx_i\rangle,
    \end{align*}
    and we can lower bound the right hand side of \cref{eq:align_perp} as
    follows:
    \begin{align}\label{eq:align_perp_rhs}
        \enVert{\Pip\nR\del{\barw(B)}}\ge\ip{-\Pip\nR\del{\barw(B)}}{\hu}
        &\ge \frac{1}{n}\sum_{(\vx_i,y_i)\in D_c}-\ell'\del{\ip{\barw(B)}{y_i\vx_i}}\hgamma.
    \end{align}
    Since $\cR_s\del{\barw(B)}\le\cR\del{\barw(B)}\le\ell(0)$, and $\cR_s$ has
    compact sublevel sets on $S$, we know that $\Pi_S\barw(B)$ is bounded.
    Consequently
    \begin{align*}
        \lim_{B\to\infty}\frac{\barw_\perp(B)}{\enVert{\barw_\perp(B)}}=\lim_{B\to\infty}\frac{\barw(B)}{B}=\baru.
    \end{align*}
    If the margin of $\baru$ on $D_c$ is less than
    $\epsilon:=\hgamma^2/\del{8|D_c|}$, then there exists $B_0$ such that for
    any $B\ge B_0$, the margin of $\barw_\perp(B)/\enVert{\barw_\perp(B)}$ on
    $D_c$ is no larger than $\epsilon$, and the distance between $\barw(B)/B$
    and $\barw_\perp(B)/\enVert{\barw_\perp(B)}$ is no larger than $\epsilon$.
    Let $H$ denote the subset of $D_c$ on which the margin of
    $\barw_\perp(B)/\enVert{\barw_\perp(B)}$ is larger than $\hgamma$, and
    suppose the minimum margin of $\barw_\perp(B)/\enVert{\barw_\perp(B)}$ on
    $D_c$ is attained at $i_1$.
    Then \cref{eq:align_perp,eq:align_perp_lhs,eq:align_perp_rhs} give
    \begin{align}
         & \ \sum_{(\vx_i,y_i)\in H}^{}-\ell'\del{\ip{\barw(B)}{y_i\vx_i}}\del{\ip{\frac{\barw_\perp(B)}{\enVert{\barw_\perp(B)}}}{y_i\vx_i}-\hgamma} \label{eq:perp_tmp1} \\
        \ge & \ \sum_{(\vx_i,y_i)\in D_c\setminus H}^{}-\ell'\del{\ip{\barw(B)}{y_i\vx_i}}\del{\hgamma-\ip{\frac{\barw_\perp(B)}{\enVert{\barw_\perp(B)}}}{y_i\vx_i}} \nonumber \\
        \ge & \ -\ell'\del{\ip{\barw(B)}{y_{i_1}\vx_{i_1}}}\del{\hgamma-\ip{\frac{\barw_\perp(B)}{\enVert{\barw_\perp(B)}}}{y_{i_1}\vx_{i_1}}}. \label{eq:perp_tmp2}
    \end{align}
    Note that by our conditions, for any $i$,
    \begin{align*}
        \envert{\ip{\frac{\barw(B)}{B}}{y_i\vx_i}-\ip{\frac{\barw_\perp(B)}{\enVert{\barw_\perp(B)}}}{y_i\vx_i}}\le\epsilon.
    \end{align*}
    Therefore \cref{eq:perp_tmp1} can be upper bounded by
    \begin{align*}
        -\ell'\del{B\del{\hgamma-\epsilon}}(1-\hgamma)|H|\le-\ell'\del{\frac{B\hgamma}{2}}|D_c|,
    \end{align*}
    while \cref{eq:perp_tmp2} can be lower bounded by
    \begin{align*}
        -\ell'\del{B(\epsilon+\epsilon)}(\hgamma-\epsilon)\ge-\ell'(2B\epsilon)\frac{\hgamma}{2}.
    \end{align*}
    Consequently, for any $z\ge\alpha:=2B_0\epsilon$,
    \begin{align*}
        -\ell'\del{\frac{\hgamma z}{4\epsilon}}\ge-\ell'(z)\frac{\hgamma}{2|D_c|}=-\ell'(z)\frac{4\epsilon}{\hgamma}.
    \end{align*}
    Similar to the proof of \Cref{fact:baru_margin}, we can show that
    $\int_\alpha^\infty-\ell(z)\dif z=\infty$, a contradiction.
\end{proofof}

\begin{proofof}{\Cref{fact:baru_dec_gen}}
    Let $\bar{\cR}:=\inf_{\vw\in\R^d}\cR(\vw)$.
    Also recall that $\barv$ denote the unique minimizer of $\cR_s$ over $S$.
    Since for any $\vw$,
    \begin{align*}
        \cR_s\del{\vw_S+(1+\alpha)\enVert{\vw_\perp}\baru}=\cR_s(\vw),
    \end{align*}
    we only need to show that
    \begin{align*}
        \cR_c\del{\vw_S+(1+\alpha)\enVert{\vw_\perp}\baru}\le\cR_c(\vw).
    \end{align*}
    Let $\xi(\alpha)$ be small enough such that for any $\vw$ with
    $\cR(\vw)-\bar{\cR}\le\xi(\alpha)$, the following properties hold.
    \begin{enumerate}
        \item $\|\barv-\vw_S\|\le1$.
        \item $\|\barv\|+\nicefrac{1}{\bgamma}\le\alpha\bgamma\|\vw_\perp\|/4\le\alpha\|\vw_\perp\|/4$.
        \item For any $B\ge\|\vw_\perp\|-\nicefrac{1}{\bgamma}$, it holds that
        $\enVert{\nicefrac{\barw(B)}{B}-\baru}\le\alpha\bgamma/4$.
    \end{enumerate}
    Consider $\vw$ which satisfies $\cR(\vw)-\bar{\cR}\le\xi(\alpha)$, and define
    \begin{align*}
        \tilde{\vw}=\vw+(\barv-\vw_S)+\frac{\|\barv-\vw_S\|}{\bar{\gamma}}\baru.
    \end{align*}
    By definition $\cR_s(\tilde{\vw})=\bar{\cR}$, and since for any
    $(\vx_i,y_i)\in D_c$ it holds that
    \begin{align*}
        y_i \langle\tilde{\vw},\vx_i\rangle & =y_i \langle \vw,\vx_i\rangle+y_i \langle\barv-\vw_S,\vx_i\rangle+\frac{\|\barv-\vw_S\|}{\bar{\gamma}}y_i \langle\baru,\vx_i\rangle \\
         & \ge y_i \langle \vw,\vx_i\rangle-\|\barv-\vw_S\|+\|\barv-\vw_S\| \\
         & =y_i \langle \vw,\vx_i\rangle,
    \end{align*}
    we have $\cR_c(\tilde{\vw})\le\cR_c(\vw)$.
    On the other hand, by definition
    $\cR\del{\barw\del{\|\tilde{\vw}\|}}\le\cR(\tilde{\vw})$, and since
    $\cR_s\del{\barw\del{\|\tilde{\vw}\|}}\ge\bar{\cR}=\cR_s(\tilde{\vw})$, we have
    \begin{align}\label{eq:baru_dec_gen_tmp1}
        \cR_c\del{\barw\del{\|\tilde{\vw}\|}}\le\cR_c(\tilde{\vw})\le\cR_c(\vw).
    \end{align}
    Note that due to bullet 1 above,
    \begin{align*}
        \|\tilde{\vw}\|\ge\|\Pi_\perp\tilde{\vw}\|\ge\|\vw_\perp\|-\frac{\|\barv-\vw_S\|}{\bgamma}\ge\|\vw_\perp\|-\frac{1}{\bgamma}.
    \end{align*}
    Therefore due to bullet 3 above $\enVert{\nicefrac{\barw\del{\|\tilde{\vw}\|}}{\|\tilde{\vw}\|}-\baru}\le\alpha\bgamma/4$, which implies for any $(\vx_i,y_i)\in D_c$,
    \begin{align}\label{eq:baru_dec_gen_tmp2}
        \del{1+\frac{\alpha}{4}}y_i \langle\|\tilde{\vw}\|\baru,\vx_i\rangle\ge y_i \ip{\barw\del{\|\tilde{\vw}\|}}{\vx_i}.
    \end{align}
    On the other hand, by the triangle inequality and bullet 1 and 2,
    \begin{align*}
        \|\tilde{\vw}\| & =\enVert{\barv+\vw_\perp+\frac{\|\barv-\vw_S\|}{\bgamma}\baru} \\
         & \le\|\barv\|+\|\vw_\perp\|+\frac{\|\barv-\vw_S\|}{\bgamma} \\
         & \le\|\barv\|+\|\vw_\perp\|+\frac{1}{\bgamma}\le\del{1+\frac{\alpha\bgamma}{4}}\|\vw_\perp\|\le\del{1+\frac{\alpha}{4}}\|\vw_\perp\|,
    \end{align*}
    and thus
    \begin{align}\label{eq:baru_dec_gen_tmp3}
        \del{1+\frac{\alpha}{4}}y_i \langle\|\tilde{\vw}\|\baru,\vx_i\rangle\le\del{1+\frac{\alpha}{4}}^2y_i \langle\|\vw_\perp\|\baru,\vx_i\rangle\le\del{1+\frac{3\alpha}{4}}y_i \langle\|\vw_\perp\|\baru,\vx_i\rangle.
    \end{align}
    Moreover, due to bullet 1 and 2,
    \begin{align}\label{eq:baru_dec_gen_tmp4}
        y_i \langle \vw_S,\vx_i\rangle\ge-\|\vw_S\|\ge-\|\barv\|-1\ge-\|\barv\|-\frac{1}{\bgamma}\ge-\frac{\alpha\bgamma\|\vw_\perp\|}{4}\ge-\frac{\alpha}{4}y_i \langle\|\vw_\perp\|\baru,\vx_i\rangle.
    \end{align}
    Combining \cref{eq:baru_dec_gen_tmp2,eq:baru_dec_gen_tmp3,eq:baru_dec_gen_tmp4}
    gives
    \begin{align*}
        y_i \ip{\barw\del{\|\tilde{\vw}\|}}{\vx_i}\le\del{1+\frac{3\alpha}{4}}y_i \langle\|\vw_\perp\|\baru,\vx_i\rangle\le y_i \ip{\vw_S+(1+\alpha)\|\vw_\perp\|\baru}{\vx_i},
    \end{align*}
    which implies
    \begin{align}\label{eq:baru_dec_gen_tmp5}
        \cR_c\del{\vw_S+(1+\alpha)\|\vw_\perp\|\baru}\le\cR_c\del{\barw\del{\|\tilde{\vw}\|}}.
    \end{align}
    It follows from \cref{eq:baru_dec_gen_tmp1,eq:baru_dec_gen_tmp5} that
    \begin{align*}
        \cR_c\del{\vw_S+(1+\alpha)\|\vw_\perp\|\baru}\le\cR_c(\vw),
    \end{align*}
    which concludes the proof.
\end{proofof}

\begin{proofof}{\Cref{fact:conv_dir_gen}(2), the ``only if'' part}
    Given any $\epsilon\in(0,1)$, let $\alpha$ satisfy
    $1/(1+\alpha)=1-\epsilon$ (i.e., let $\alpha=\epsilon/(1-\epsilon)$).

    Since $\lim_{t\to\infty}\cR(\vw_t)=\inf_{\vw\in\R^d}\cR(\vw)$, there exists $t_0$
    such that for any $t\ge t_0$ we have
    $\cR(\vw_t)-\inf_{\vw\in\R^d}\cR(\vw)\le\xi(\alpha)$ and $\|\vw_{t,\perp}\|\ge1$.
    By convexity and \Cref{fact:baru_dec_gen}, for $t\ge t_0$,
    \begin{align*}
        \ip{\nR(\vw_t)}{\vw_{t,\perp}-(1+\alpha)\enVert{\vw_{t,\perp}}\baru} & =\ip{\nR(\vw_t)}{\vw_{t,S}+\vw_{t,\perp}-\vw_{t,S}-(1+\alpha)\enVert{\vw_{t,\perp}}\baru} \\
         & =\ip{\nR(\vw_t)}{\vw_t-\vw_{t,S}-(1+\alpha)\enVert{\vw_{t,\perp}}\baru} \\
         & \ge\cR(\vw_t)-\cR\del{\vw_{t,S}+(1+\alpha)\enVert{\vw_{t,\perp}}\baru}\ge0.
    \end{align*}
    Consequently,
    \begin{align*}
      \langle \vw_{t+1}-\vw_t,\baru\rangle & =\ip{-\eta\nR(\vw_t)}{\baru} \\
       & \ge\ip{-\eta\nR(\vw_t)}{\vw_{t,\perp}}\frac{1}{(1+\alpha)\|\vw_{t,\perp}\|} \\
       & =\langle \vw_{t+1}-\vw_t,\vw_{t,\perp}\rangle \frac{1}{(1+\alpha)\|\vw_{t,\perp}\|} \\
       & =\langle \vw_{t+1,\perp}-\vw_{t,\perp},\vw_{t,\perp}\rangle \frac{1}{(1+\alpha)\|\vw_{t,\perp}\|} \\
       & =\del{\frac{1}{2}\|\vw_{t+1,\perp}\|^2-\frac{1}{2}\|\vw_{t,\perp}\|^2-\frac{1}{2}\|\vw_{t+1,\perp}-\vw_{t,\perp}\|^2}\frac{1}{(1+\alpha)\|\vw_{t,\perp}\|}.
    \end{align*}
    On one hand, we have
    \begin{align*}
      \del{\frac{1}{2}\|\vw_{t+1,\perp}\|^2-\frac{1}{2}\|\vw_{t,\perp}\|^2}/\|\vw_{t,\perp}\|\ge\|\vw_{t+1,\perp}\|-\|\vw_{t,\perp}\|.
    \end{align*}
    On the other hand, using the step size condition in \cref{eq:gd_risk},
    we have
    \begin{align*}
      \frac{\|\vw_{t+1,\perp}-\vw_{t,\perp}\|^2}{2(1+\alpha)\|\vw_{t,\perp}\|}\le \frac{\|\vw_{t+1,\perp}-\vw_{t,\perp}\|^2}{2} & \le \frac{\|\vw_{t+1}-\vw_t\|^2}{2} \\
       & =\frac{\eta^2\enVert{\nR(\vw_t)}^2}{2} \\
       & \le\eta\del{\cR(\vw_t)-\cR(\vw_{t+1})}.
    \end{align*}
    As a result,
    \begin{align*}
      \langle \vw_t-\vw_{t_0},\baru\rangle\ge \frac{\|\vw_{t,\perp}\|-\|\vw_{t_0,\perp}\|}{1+\alpha}-\eta\cR(\vw_{t_0})=\del{1-\epsilon}\del{\|\vw_{t,\perp}\|-\|\vw_{t_0,\perp}\|}-\eta\cR(\vw_{t_0}),
    \end{align*}
    meaning
    \begin{align*}
        \ip{\frac{\vw_t}{\|\vw_t\|}}{\baru}\ge(1-\epsilon)\frac{\|\vw_{t,\perp}\|}{\|\vw_t\|}+\frac{\langle \vw_{t_0},\baru\rangle-(1-\epsilon)\|\vw_{t_0,\perp}\|-\eta\cR(\vw_{t_0})}{\|\vw_t\|}.
    \end{align*}
    Consequently,
    \begin{align*}
      \liminf_{t\to\infty}\ip{\frac{\vw_t}{\|\vw_t\|}}{\baru}\ge1-\epsilon.
    \end{align*}
    Since $\epsilon$ is arbitrary, we get $\vw_t/\|\vw_t\|\to\baru$.
\end{proofof}
\end{document}